\documentclass[final]{IEEEtran}

\usepackage{amsmath}
\usepackage{epsfig}
\usepackage{amsthm}
\usepackage{url}
\usepackage{multirow}
\usepackage{rotating}
\usepackage{amssymb}
\usepackage{graphicx,amsfonts}
\usepackage{algorithm,algorithmic}
\usepackage{color}
\usepackage{bbm} %double structed letter like mathbb

\usepackage{cite}
\ifCLASSINFOpdf
\else
\fi

\usepackage{stfloats}

\theoremstyle{my_style}

\newtheorem{definition}{Definition}[section]
\newtheorem{remark}{Remark}[section]
\newtheorem{lemma}{Lemma}[section]

\newcommand{\cH}{\mathcal{H}}
\newcommand{\cX}{\mathcal{X}}
\newcommand{\cL}{\mathcal{L}}

\newcommand{\bx}{\boldsymbol{x}}

\newcommand{\by}{\boldsymbol{y}}

\newcommand{\bz}{\boldsymbol{z}}
\newcommand{\ba}{\boldsymbol{a}}
\newcommand{\bb}{\boldsymbol{b}}

\newcommand{\bw}{\boldsymbol{w}}

\newcommand{\bbf}{\boldsymbol{f}}
\newcommand{\bg}{\boldsymbol{g}}

\newcommand{\beq}{\begin{equation}}
\newcommand{\eeq}{\end{equation}}
\newcommand{\beqn}{\begin{eqnarray}}
\newcommand{\eeqn}{\end{eqnarray}}
\newcommand{\beqns}{\begin{eqnarray*}}
\newcommand{\eeqns}{\end{eqnarray*}}

\newcommand{\ii}{i}

\newcommand{\R}{\mathbb{R}}
\newcommand{\HH}{\mathbb{H}}
\newcommand{\XX}{\mathbb{X}}
\newcommand{\C}{\mathbb{C}}

\newcommand{\F}{\mathbb{F}}
\newcommand{\N}{\mathbb{N}}

\makeatletter
\newcommand{\bdiv}{\mathop{\operator@font div}}
\makeatother

\makeatletter
\newcommand{\diag}{\mathop{\operator@font diag}}
\makeatother

\makeatletter
\newcommand{\conv}{\mathop{\operator@font conv}}
\makeatother

\makeatletter
\newcommand{\minim}{\mathop{\operator@font minimize}}
\makeatother

\makeatletter
\newcommand{\maxim}{\mathop{\operator@font maximize}}
\makeatother

\makeatletter
\newcommand{\sign}{\mathop{\operator@font sign}}
\makeatother \makeatletter

\makeatletter
\newcommand{\proj}{\mathop{\operator@font proj}}
\makeatother \makeatletter

\makeatletter
\newcommand{\spa}{\mathop{\operator@font span}}
\makeatother \makeatletter

\makeatletter
\newcommand{\epi}{\mathop{\operator@font epi}}
\makeatother \makeatletter

\makeatletter
\newcommand{\dom}{\mathop{\operator@font dom}}
\makeatother \makeatletter

\makeatletter
\newcommand{\real}{\mathop{\operator@font Re}}
\makeatother \makeatletter

\makeatletter
\newcommand{\imag}{\mathop{\operator@font Im}}
\makeatother \makeatletter

\makeatletter
\newcommand{\sinc}{\mathop{\operator@font sinc}}
\makeatother \makeatletter

\begin{document}
\title{Complex Support Vector Machines for Regression and Quaternary Classification}

\author{Pantelis Bouboulis,~\IEEEmembership{Member,~IEEE,}
        and~Sergios Theodoridis,~\IEEEmembership{Fellow,~IEEE,}
        and~Charalampos Mavroforakis,
        and Leoni Dalla
\thanks{Copyright (c) 2013 IEEE. Personal use of this material is permitted.
However, permission to use this material for any other purposes must be obtained from the
IEEE by sending a request to pubs-permissions@ieee.org.}
\thanks{This research has been co-financed by the European Union (European Social
Fund – ESF) and Greek national funds through the Operational Program
''Education and Lifelong Learning'' of the National Strategic Reference
Framework (NSRF) - Research Funding Program: Aristeia I: 621.}
\thanks{P. Bouboulis is with the Department
of Informatics and Telecommunications, University of Athens, Greece,
e-mail: panbouboulis@gmail.com.}% <-this % stops a space
\thanks{S. Theodoridis is with the Department
of Informatics and Telecommunications, University of Athens, Greece,
and the Research Academic Computer Technology Institute, Patra, Greece.
e-mail: stheodor@di.uoa.gr.}%
\thanks{Ch. Mavroforakis is with the Department of Computer Science,
       Data Management Lab, Boston University, Boston, MA 02215, USA.
       e-mail: cmav@bu.edu.}%
\thanks{L. Dalla is with the Department of Mathematics,
       University of Athens, Greece.
       email: ldalla@math.uoa.gr.}%
}


%Telephone: (800) 555--1212, Fax: (888) 555--1212}


\maketitle

\begin{abstract}
The paper presents a new framework for complex Support Vector Regression as well as Support Vector Machines for quaternary classification.  The method exploits the notion of widely linear estimation to model the input-out relation for complex-valued data and considers two cases: a) the complex data are split into their real and imaginary parts and a typical real kernel is employed to map the complex data to a complexified feature space and b) a pure complex kernel is used to directly map the data to the induced complex feature space.  The recently developed Wirtinger's calculus on complex reproducing kernel Hilbert spaces (RKHS) is employed in order to compute the Lagrangian and derive the dual optimization problem. As one of our major results, we prove that any complex SVM/SVR task is equivalent with solving two real SVM/SVR tasks exploiting a specific real kernel which is generated by the chosen complex kernel. In particular, the case of pure complex kernels leads to the generation of new kernels, which have not been considered before. In the classification case, the proposed framework inherently splits the complex space into four parts. This leads naturally in solving the four class-task (quaternary classification), instead of the typical two classes of the real SVM. In turn, this rationale can be used in a multiclass problem as a split-class scenario based on four classes, as opposed to the one-versus-all method; this can lead to significant computational savings.
Experiments demonstrate the effectiveness of the proposed framework for regression and classification tasks that involve complex data.
%In many cases, the exploitation of complex data improves the classification accuracy, or reduces the mean square error in regression tasks. Moreover, it is shown that in some cases the use of pure complex kernels leads to superior performance compared to the typical case where the complex data are split into two real parts.
\end{abstract}


% \fi
\IEEEpeerreviewmaketitle

\section{Introduction}\label{SEC:INTRO}
The SVM framework has become a popular toolbox for addressing real world applications that involve non-linear classification and regression tasks. In its original form, the SVM method is a nonlinear generalization of the \textit{Generalized Portrait} algorithm, which has been developed in the former USSR in the 1960s. The introduction of non-linearity was carried out via a computationally elegant way known today as the \textit{kernel trick} \cite{Scholkopf_2002_2276}. Usually, this trick is applied in a black-box rationale, where one simply replaces dot products with a positive definite kernel function.
The successful application of the kernel trick in SVMs  has sparked a new breed of technics for addressing non linear tasks, the so called \textit{kernel-based methods}. Currently, kernel-based algorithms constitute a popular tool employed in a variety of scientific domains, ranging from adaptive filtering \cite{Eref2, Muller01} and image processing to biology and nuclear physics \cite{Scholkopf_2002_2276, Theodoridis_2008_9253, Vapnik_1999_9348, ShaweTaylor_2004_9255,  Liu_2010_10644, Kivinen_2004_11230, Engel_2004_11231, Slavakis_2008_9257, Liu_2008_10645, Slavakis_2009_10648, Mavroforakis_2007_4002, Theodoridis_2007_9346, Bouboulis_2010_10642, Bouboulis_2011_10643, Schlkopf_2004_11445, Muller}. The key mathematical notion underlying these methods is that of RKHS. These are inner product spaces in which the pointwise evaluation functional is continuous. Through the kernel trick, the original data are transformed into a higher dimensional RKHS $\cH$ (possibly of infinite dimension) and linear tools are applied to the transformed data in the so called feature space $\cH$. This is equivalent to solving a non-linear problem in the original space. Furthermore, inner products in $\cH$ can efficiently be computed via the specific kernel function $\kappa$ associated with the RKHS $\cH$, disregarding the actual structure of the space. Recently, this rationale has been generalized, so that the task simultaneously learns the so called kernel in some fashion, instead of selecting it a priori, in the context of multiple kernel learning (MKL) \cite{Bach04a, Bach08a, Gonen, Sonnenburg}.

Although  the theory of RKHS has been developed by mathematicians for general complex spaces, most kernel-based methods employ real kernels. This is largely due to the fact that many of them originated as variants of the original SVM formulation, which was targeted to treat real data. However, in modern applications, complex data arise frequently in areas as diverse as communications, biomedicine, radar, etc. Although all the respective algorithms that employ complex data (e.g., in communications) can also be cast in the real domain (disregarding any type of complex algebra), by splitting the complex data into two parts and working separately,  this approach usually leads to more intricate expressions and tedious calculations.  The complex domain provides a convenient and elegant representation for such data, but also a natural way to preserve their characteristics and to handle transformations that need to be performed.

Hence, the design of SVMs suitable for treating problems of complex and/or multidimensional outputs has attracted some attention in the machine learning community. Perhaps the most complete works, which attempt to generalize the SVM rationale in this fashion, are a) Clifford SVM \cite{BayroCorrochano_2010_10641} and b) division algebraic SVR \cite{Shilton2007,Shilton2010,Shilton2012}. In Clifford SVM, the authors use Clifford algebras to extend the SVM framework to multidimensional outputs. Clifford algebras belong to a type of associative algebras, which are used in mathematics to generalize the complex numbers, quaternions and several other hypercomplex number systems. On the other hand, in division algebraic SVR, division algebras are employed for the same purpose. These are algebras, closely related to the Clifford algebras, where all non-zero elements have multiplicative inverses. In a nutshell, Clifford algebras are more general and they can be employed to create a general algebraic framework (i.e., addition and multiplication operations) in any type of vector spaces (e.g., $\R$, $\R^2$, $\R^3$, $\dots$), while the division algebras are only four: the real numbers, the complex numbers ($\R^2$), the quaternions ($\R^4$) and the octonions ($\R^8$). This is due to the fact that the need for inverses can only be satisfied in these four vector spaces. Although Clifford algebras are more general, their limitations (e.g., the lack of inverses) make them a difficult tool to work with, compared to the division algebras. Another notable attempt that pursues similar goals is the multiregression SVMs of \cite{SVM_multiregression}, where the outputs are represented simply as vectors and an $\epsilon$-insensitive loss function is adopted. Unfortunately this approach does not result in a well defined dual problem.
In contrast to the more general case of hyper-complex outputs, where applications are limited \cite{Ujang}, complex valued SVMs have been adopted by a number of authors for the beamforming problem (e.g., \cite{BeamSVM1, BeamSVM2}), although restricted to the simple linear case.

It is important to emphasize that most of the aforementioned efforts to apply the SVM rationale to complex and hypercomplex numbers are limited to the case of the output data\footnote{In \cite{Corrochano2010} the authors consider also a  Gabor kernel function which takes multivector inputs.}. These methods consider a multidimensional output, which can be represented, for example, as a complex number or a quaternion, while the input data are real vectors. In some cases, complex input data are considered as well, but in a rather trivial way, i.e., splitting the data into their real and imaginary parts. Moreover, these methods employ real valued kernels to model the input-out relationship, breaking it down to its multidimensional components. However, in this way many of the rich geometric characteristics of complex and hypercomplex spaces are lost.

In this paper, we adopt an alternative rationale. To be in line with the current trend in complex signal processing, we employ the so-called widely linear estimation process, which has been shown to perform better than the conventional linear estimation process \cite{Boub_ACKLMS, Adali_2010_10640, Novey_2008_10638, Mandic_2009_10646, Kuh09}. This means that we model the input-out relationship as a sum of  two parts. The first is linear with respect to the input vector, while the second is linear with respect to its conjugate. Furthermore, we consider two cases to generalize the SVM framework to complex spaces. In the first one, the data are split into their real and imaginary parts and typical well established real kernels are employed to map the data into a complexified RKHS. This scenario bears certain similarities with other preexisting technics that also split the output into two parts (e.g., \cite{Shilton2010}). The difference with our technique is that the widely linear estimation process is employed to model the input-out relationship of the SVM.  In the second case, the modeling takes place directly into complex RKHS, which are generated by pure complex kernels\footnote{The term ``pure complex kernels'' refers to complex valued kernels with complex variables, that are complex analytic.}, instead of real ones. In that fashion, the geometry of the complex space is preserved.  Moreover, we show that in the case of complex SVMs, the widely linear approach is a \textit{necessity}, as the alternative path would lead to a significantly restricted model. In order to compute the gradients, which are required by the Karush-Kuhn-Tucker (KKT) conditions and the dual, we employ the generalized Wirtinger Calculus introduced in \cite{Bouboulis_2011_10643}.
As one of our major results, we prove that working in a complex RKHS $\HH$, with a pure complex kernel $\kappa_{\C}$, is equivalent to solving two problems in a real RKHS $\cH$, albeit with a specific real kernel $\kappa_{\R}$, which is induced by the complex $\kappa_{\C}$. It must be pointed out that these induced kernels are not trivial. For example, the exploitation of the complex Gaussian kernel results in an induced kernel different from the standard real Gaussian RBF.

To summarize, the main contribution of our work is the development of a complete mathematical framework suitable for treating any SVR/SVM task, that involves complex data, in an elegant and uniform manner. Moreover, we provide a new way of treating a special multi-classification problem (i.e., quaternary classification).
Our emphasis in this paper is to outline the theoretical development and to verify the validity of our results via some simulation examples. The paper is organized as follows: In Section \ref{SEC:RKHS} the main mathematical background regarding RKHS is outlined and the differences between a real RKHS and a complex RKHS are highlighted. The main contributions of the paper can be found in Sections \ref{SEC:CSVR} and \ref{SEC:CSVM}, where the theory and the generalized complex SVR and SVM algorithms are developed, respectively. The complex SVR developed there, is suitable for general complex valued function estimation problems defined on complex domains. The proposed complex SVM rationale, on the other hand, is suitable for quaternary (i.e., four class) classification, in contrast to the binary classification carried out by the real SVM approach.  The experiments that are presented in Section \ref{SEC:EXP} demonstrate certain cases where the use of the pure complex Gaussian kernel in the SVR rationale offers significant advantages over the real Gaussian kernel. In the SVM, besides the new case of quaternary classification, experiments also show how the exploitation of complex data improves the classification accuracy. Finally, Section \ref{SEC:CONCL} contains some concluding remarks.

\section{Real and Complex RKHS}\label{SEC:RKHS}
We devote this section to present the notation that is adopted in the paper and to summarize the basic mathematical background regarding RKHS. Throughout the paper, we will denote the set of all integers, real and complex numbers by $\N$, $\R$ and $\C$, respectively. The imaginary unit is denoted as $\ii$, while $z^*$ denotes the conjugate of $z$. Vector or matrix valued quantities appear in boldfaced symbols.

An RKHS \cite{Aronszajn_1950_9268} is a Hilbert space $\cH$ over a field $\F$ for which there exists a positive definite kernel function $\kappa:\cX\times \cX\rightarrow\F$ with the following two important properties: a) For every $x\in \cX$, $\kappa(\cdot,x)$ belongs to $\cH$ and b) $\kappa$ has the so called \textit{reproducing property}, i.e.,
$f(x)=\langle f,\kappa(\cdot, x)\rangle_\cH, \textrm{ for all } f\in\cH$,
in particular $\kappa(x,y)=\langle \kappa(\cdot, y), \kappa(\cdot, x)\rangle_\cH$.
The map $\Phi:\cX\rightarrow\cH:\Phi(x)=\kappa(\cdot,x)$ is called the \textit{feature map} of $\cH$.  In the case of complex spaces (i.e., $\F=\C$) the inner product is sesqui-linear (i.e., linear in one argument and antilinear in the other)  and Hermitian, i.e., $\kappa(x,y)=\left(\langle \kappa(\cdot, x), \kappa(\cdot, y)\rangle_\cH\right)^* = \kappa^*(y,x)$. In the real case, however, this is simplified to $\kappa(x,y)=\langle \kappa(\cdot, y), \kappa(\cdot, x)\rangle_\cH=\langle \kappa(\cdot, x), \kappa(\cdot, y)\rangle_\cH$. In the following, we will denote by $\HH$ a complex RKHS and by $\cH$ a real RKHS. Moreover, in order to distinguish the two cases, we will use the notations $\kappa_{\R}$ and $\Phi_{\R}$ to refer to a real kernel and its corresponding feature map, instead of the notation $\kappa_{\C}$, $\Phi_{\C}$, which is reserved for pure complex kernels.

A variety of kernel functions can be found in the respective literature \cite{Scholkopf_2002_2276, Theodoridis_2008_9253, ShaweTaylor_2004_9255, Bouboulis_2011_11457, Paulsen_2009_11235}. In this paper we will use the popular \textit{real Gaussian kernel}, i.e., $\kappa_{\R^\nu,t}(\bx,\by) : = \exp\left(-t\sum_{k=1}^{\nu}(x_k-y_k)^2\right)$,
defined for $\bx, \by \in \R^\nu$, and the \textit{complex Gaussian kernel}, i.e.,
$\kappa_{\C^\nu,t}(\bz,\bw) : = \exp\left(-t\sum_{k=1}^{\nu}(z_k-w_k^*)^2\right)$,
where $\bz,\bw\in\C^\nu$, $z_k$ denotes the $k$-th component of the complex vector $\bz\in\C^\nu$ and $\exp(\cdot)$ is the extended exponential function in the complex domain. In both cases $t$ is a free positive parameter that controls the shape of the kernel.

Besides the complex RKHS produced by the associated complex kernels, such as the aforementioned ones, one may construct a complex RKHS as a Cartesian product of a real RKHS with itself, in a fashion similar to the identification of the field of complex numbers, $\C$, to $\R^2$. This technique is called \textit{complexification} of a real RKHS and the respective Hilbert space is called \textit{complexified} RKHS.  Let $\cX\subseteq\R^\nu$ and define the spaces $\cX^2\equiv \cX\times \cX\subseteq\R^{2\nu}$ and $\XX=\{\bx+\ii\by; \bx,\by\in \cX\}\subseteq\C^{\nu}$, where the latter is equipped with a complex inner product structure. Let $\cH$ be a real RKHS associated with a real kernel $\kappa_{\R}$ defined on $\cX^2\times \cX^2$ and let $\langle\cdot,\cdot\rangle_\cH$ be its corresponding inner product. Then, every $f\in\cH$ can be regarded as a function defined on either $\cX^2$ or $\XX$, i.e., $f(\bz) = f(\bx+\ii\by) = f(\bx,\by)$.
Moreover, we define the Cartesian product of $\cH$ with itself, i.e., $\cH^2=\cH\times\cH$. It is easy to verify that $\cH^2$ is also a Hilbert space with inner product
\begin{align}
\langle \bbf, \bg\rangle_{\cH^2} = \langle f^r, g^r\rangle_\cH + \langle f^i, g^i\rangle_\cH,
\end{align}
for $\bbf=(f^r,f^i)$, $\bg=(g^r,g^i)$. Our objective is to enrich $\cH^2$ with a complex structure (i.e., with a complex inner product). To this end, we define the space $\HH= \{f=f^r + \ii f^i;\;f^r,f^i\in\cH\}$
equipped with the complex inner product:
\begin{align}\label{EQ:complex_inner}
\langle f, g\rangle_{\HH}= \langle f^r, g^r\rangle_\cH + \langle f^i, g^i\rangle_\cH +
                 \ii\left(\langle f^i, g^r\rangle_\cH - \langle f^r, g^i\rangle_\cH\right),
\end{align}
for $f=f^r + \ii f^i$, $g=g^r + \ii g^i$. It is not difficult to verify that the complexified space $\HH$ is a complex RKHS with kernel $\kappa$ \cite{Paulsen_2009_11235}. We call $\HH$ the complexification of $\cH$. It can readily be seen, that, although $\HH$ is a complex RKHS, its respective kernel is real (i.e., its imaginary part is equal to zero). To complete the presentation of the complexification procedure, we need a technique to implicitly map the data samples from the complex input space to the complexified RKHS $\HH$. This can be done using the simple rule:
\begin{align}\label{EQ:Phi_map}
\begin{array}{ll}
\bar\Phi_{\C}(\bz)&=\bar\Phi_{\C}(\bx+ \ii\by) = \bar\Phi_{\C}(\bx,\by)\\
&= \Phi_{\R}(\bx,\by) + \ii\Phi_{\R}(\bx,\by),
\end{array}
\end{align}
where $\Phi_{\R}$ is the feature map of the real reproducing kernel $\kappa_{\R}$, i.e., $\Phi_{\R}(\bx,\by)=\kappa_{\R}(\cdot, (\bx, \by))$ and $\bz = \bx + \ii\by$. As a consequence, observe that:
\begin{align*}
\langle\bar\Phi_{\C}(\bz), \bar\Phi_{\C}(\bz')\rangle_{\HH} &= 2\langle\Phi_{\R}(\bx,\by), \Phi_{\R}(\bx',\by')\rangle_{\cH}\\
 &= 2\kappa_{\R}( (\bx',\by'), (\bx,\by) ),
\end{align*}
for all $\bz, \bz'\in\HH$.
We have to emphasize that a complex RKHS $\HH$ (whether it is constructed through the complexification procedure, or it is produced by a complex kernel) can, always, be represented as a Cartesian product of a Hilbert space with itself, i.e., we can, always, identify $\HH$ with a \textit{double real space} $\cH^2$. Furthermore, the complex inner product of $\HH$ can always be related to the real inner product of $\cH$ as in (\ref{EQ:complex_inner}).

In order to compute the gradients of real valued cost functions, which are defined on complex domains, we adopt the rationale of Wirtinger's calculus \cite{Wirtinger_1927_10651}. This was brought into light recently \cite{Adali_2010_10640,  Novey_2008_10638, Li_2008_10664}, as a means to compute, in an efficient and elegant way,  gradients of real valued cost functions that are defined on complex domains ($\C^\nu$), in the context of widely linear processing \cite{Mandic_2009_10646, Picinbono_1995_10647}. It is based on simple rules and principles, which bear a great resemblance to the rules of the standard complex derivative, and it greatly simplifies the calculations of the respective derivatives. The difficulty with real valued cost functions is that they do not obey the Cauchy-Riemann conditions and are not differentiable in the complex domain. The alternative to Wirtinger's calculus would be  to consider the complex variables as pairs of two real ones and employ the common real partial derivatives. However, this approach, usually, is more time consuming and leads to more cumbersome expressions. In \cite{Bouboulis_2011_10643}, the notion of Wirtinger's calculus was extended to general complex Hilbert spaces, providing the tool to compute the gradients that are needed to develop kernel-based algorithms for treating complex data. In \cite{Bouboulis_2012_1} the notion of Wirtinger calculus was extended to include subgradients in RKHS.

\section{Complex Support Vector Regression}\label{SEC:CSVR}

We begin the treatment of the complex case with the complex SVR rationale, as this is a direct generalization of the real SVR.
Suppose we are given training data of the form $\{(\bz_n, d_n);\;n=1,\dots,N\}\subset\cX\times\C$, where $\cX=\C^\nu$ denotes the space of input patterns. As $\bz_n$ is complex, we denote by $\bx_n$ its real part and by $\by_n$ its imaginary part respectively, i.e., $\bz_n = \bx_n + \ii\by_n$, $n=1,\dots, N$. Similarly, we denote by $d^r_n$ and $d_n^i$ the real and the imaginary part of $d_n$, i.e., $d_n = d^r_n + \ii d^i_n$, $n=1,\dots,N$.

\subsection{Dual Channel SVR}\label{SEC:DRC}
A straightforward approach for addressing this problem (as well as any problem related with complex data) is by considering two different problems in the real domain. This technique is usually referred to as the \textit{dual real channel  (DRC) approach}\cite{Mandic_2009_10646}. That is, the training data are split into two sets $\{((\bx_n, \by_n)^T, d^r_n);\;n=1,\dots,N\}\subset\R^{2\nu}\times\R$ and $\{((\bx_n, \by_n)^T, d^i_n);\;n=1,\dots,N\}\subset\R^{2\nu}\times\R$, and a support vector regression is performed on each set of data using a real kernel $\kappa_{\R}$ and its corresponding RKHS. We will show in the following Sections that the DRC approach is equivalent to the complexification procedure \cite{Bouboulis_2011_10643} described in Section \ref{SEC:RKHS}. The latter, however, often provides a context that enables us to work with complex data compactly and elegantly, as one may employ Wirtinger calculus to compute the respective gradients and develop algorithms directly in complex form \cite{Bouboulis_2011_10643}.

In contrast to the complexification procedure, we emphasize that the pure complex approach (where one directly exploits a complex RKHS) considered in the next subsection is quite different from the DRC rationale. We will develop a framework for solving such a problem on the complex domain employing pure complex kernels, instead of real ones. Nevertheless, we will show that using complex kernels for SVR is equivalent with solving two real problems using a real kernel. This kernel, however, is induced by the selected complex kernel and {\it it is not one of the standard kernels} appearing in machine learning literature. For example, the use of the complex Gaussian kernel induces a real kernel, which is not the standard real Gaussian RBF (see Figure \ref{FIG:feature_map}). We demonstrated in \cite{Bouboulis_2012_1, Boub_ACKLMS}, although in a different context than the one we use here, that the DRC approach and the pure complex approaches give, in general, different results.
Depending on the case, the pure complex approach might show increased performance over the DRC approach and vice versa.

\begin{figure}[t]
\begin{center}
\includegraphics[scale=0.7]{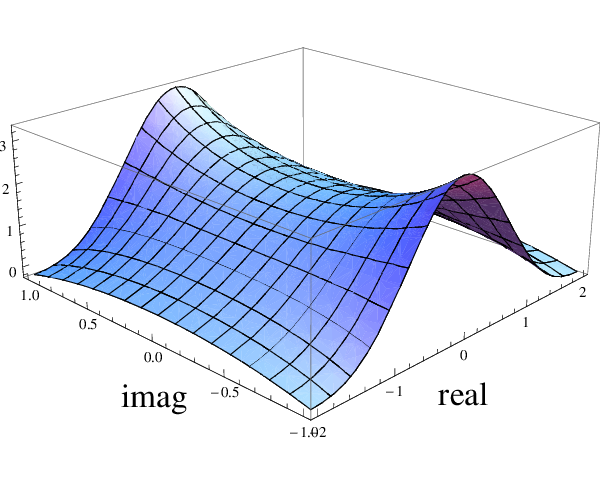}
\end{center}
\caption{The function $\kappa^r_{\C}((\cdot,\cdot)^T,(0,0)^T)$ of the induced real feature space of the complex Gaussian kernel.}\label{FIG:feature_map}
\end{figure}

\subsection{Pure Complex SVR}\label{SEC:complex_SVR}
Prior to the development of the generalized complex SVR rationale, we investigate some significant properties of the complex kernels. In the following, we assume that $\HH$ is a complex RKHS with kernel $\kappa_{\C}$. We can decompose $\kappa_{\C}$ into its real and imaginary parts, i.e., $\kappa_{\C}(\bz,\bz') = \kappa^r_{\C}(\bz,\bz') + \ii \kappa^i_{\C}(\bz,\bz')$,
where $\kappa^r_{\C}(\bz,\bz'), \kappa^i_{\C}(\bz,\bz')\in\R$.
As any complex kernel is Hermitian (see Section \ref{SEC:RKHS}), we have that $\kappa^*_{\C}(\bz,\bz') = \kappa_{\C}(\bz',\bz)$ and hence we take
\begin{align}
\kappa^r_{\C}(\bz,\bz') &= \kappa^r_{\C}(\bz',\bz), \label{EQ:kernel_real_part}\\
\kappa^i_{\C}(\bz,\bz') &= -\kappa^i_{\C}(\bz',\bz). \label{EQ:kernel_imag_part}
\end{align}

\begin{lemma}\label{LEM:imaginary_part_of_complex_kernel}
The imaginary part of any complex kernel, $\kappa_{\C}$, satisfies:
\begin{align}\label{EQ:imag_part_kernel}
\displaystyle{\sum_{n,m=1}^N} c_n c_m \kappa^i_{\C}(\bz_n,\bz_m) = 0,
\end{align}
for any $N>0$ and any selection of $c_1,\dots,c_N\in\C$ and $\bz_1,\dots, \bz_N\in\cX$.
\end{lemma}
\begin{proof}
Exploiting (\ref{EQ:kernel_imag_part}) and rearranging the indices of the summation we get:
\begin{align*}
\displaystyle{\sum_{n,m=1}^N} c_n c_m \kappa^i_{\C}(\bz_n,\bz_m) &= - \displaystyle{\sum_{n,m=1}^N} c_n c_m \kappa^i_{\C}(\bz_m,\bz_n)\\
&= - \displaystyle{\sum_{m,n=1}^N} c_m c_n \kappa^i_{\C}(\bz_n,\bz_m).
\end{align*}
Hence, $2\displaystyle{\sum_{n,m=1}^N} c_n c_m \kappa^i_{\C}(\bz_n,\bz_m) = 0$
and the result follows immediately.
\end{proof}

\begin{lemma}\label{LEM:induced_kernel}
If $\kappa_{\C}(\bz,\bz')$ is a complex kernel defined on $\C^\nu\times\C^\nu$, then its real part, i.e.,
\begin{align}\label{EQ:induced_kernel}
\kappa^{r}_{\C}\left(\left(\begin{matrix}\bx\cr\by\end{matrix}\right),\left(\begin{matrix}\bx'\cr\by'\end{matrix}\right)\right) = \real(\kappa_{\C}(\bz,\bz')),
\end{align}
where $\bz=\bx+\ii\by$, $\bz'=\bx'+\ii\by'$, is a real kernel defined on $\R^{2\nu}\times\R^{2\nu}$. We call this kernel the induced real kernel of $\kappa_{\C}$.
\end{lemma}
\begin{proof}
As relation (\ref{EQ:kernel_real_part}) implies, $\kappa^r_{\C}$ is symmetric. Moreover, let $N>0$, $\alpha_1,\dots,\alpha_N\in\R$ and $\bz_1,\dots, \bz_N\in\cX$. As $\kappa_{\C}$ is positive definite, we have that
\begin{align*}
\displaystyle{\sum_{n,m=1}^N} \alpha_n \alpha_m \kappa_{\C}(\bz_n,\bz_m) \geq 0.
\end{align*}
However, splitting $\kappa_{\C}$ to its real and imaginary parts and exploiting Lemma \ref{LEM:imaginary_part_of_complex_kernel}, we take
\begin{align*}
\displaystyle{\sum_{n,m=1}^N} \alpha_n \alpha_m \kappa_{\C}(\bz_n,\bz_m) = \displaystyle{\sum_{n,m=1}^N} \alpha_n \alpha_m \kappa^r_{\C}(\bz_n,\bz_m)\\
 + \ii\displaystyle{\sum_{n,m=1}^N} \alpha_n \alpha_m \kappa^i_{\C}(\bz_n,\bz_m)
 = \displaystyle{\sum_{n,m=1}^N} \alpha_n \alpha_m \kappa^r_{\C}(\bz_n,\bz_m).
\end{align*}
Hence, $\displaystyle{\sum_{n,m=1}^N} \alpha_n \alpha_m \kappa^r_{\C}(\bz_n,\bz_m) \geq 0$.

As a last step, recall that $\kappa^r_{\C}$ may be regarded as defined either on $\C^\nu\times\C^\nu$ or $\R^{2\nu}\times\R^{2\nu}$. This leads to
\begin{align*}
\displaystyle{\sum_{n,m=1}^N} \alpha_n \alpha_m \kappa^r_{\C}\left( \left(\begin{matrix} \bx_n \cr \by_n \end{matrix}\right), \left(\begin{matrix} \bx_m \cr \by_m \end{matrix}\right)\right) \geq 0.
\end{align*}
We conclude that $\kappa^r_{\C}$ is a positive definite kernel on $\R^{2\nu}\times\R^{2\nu}$.
\end{proof}

At this point, we are ready to present the SVR rationale in complex RKHS. We transform the input data from $\cX$ to $\HH$, via the feature map $\Phi_{\C}$, to obtain the data $\{(\Phi_{\C}(\bz_n), d_n);\;n=1,\dots,N\}$.  In analogy with the real case and extending the principles of widely linear estimation to complex support vector regression, the goal is to find a function $T:\HH\rightarrow\C: T(f) = \langle f, w\rangle_{\HH} + \langle f^*, v\rangle_{\HH} + c$, for some $u,v\in\HH$, $c\in\C$, which is as flat as possible and has at most $\epsilon$ deviation from both the real and imaginary parts of the actually obtained values $d_n$, for all $n=1,\dots,N$. We emphasize that we employ the widely linear estimation function $S_1:\HH\rightarrow\C: S_1(f) = \langle f, w\rangle_{\HH} + \langle f^*, v\rangle_{\HH}$ instead of the usual complex linear function\footnote{All other attempts to generalize the SVR rationale to complex and hypercomplex spaces employed the standard complex linear function $S_2$.} $S_2:\HH\rightarrow\C: S_2(f) = \langle f, w\rangle_{\HH}$ following the ideas of \cite{Picinbono_1995_10647}, which are becoming popular in complex signal processing \cite{Took_2010_10650, Chevalier_2006_10661, Cacciapuoti_2008_10667} and have been generalized for the case of complex RKHS in \cite{Bouboulis_2012_1}. It has been established \cite{Picinbono_1994_10659, Picinbono_1997_10660}, that the widely linear estimation functions are able to capture the second order statistical characteristics of the input data, which are necessary if non-circular\footnote{Note that the issue of circularity has become quite popular recently in the context of complex adaptive filtering. Circularity is intimately related to rotation in the geometric sense. A complex random variable $Z$ is called circular, if for any angle $\phi$ both $Z$ and $Z e^{\ii\phi}$ (i.e., the rotation of $Z$ by angle $\phi$) follow the same probability distribution \cite{Mandic_2009_10646}.} input sources are considered.  Furthermore, as it has been shown in \cite{Boub_ACKLMS}, the exploitation of the traditional complex linear function excludes a significant percentage of linear functions from being considered in the estimation process. The correct and natural linear estimation in complex spaces is the widely linear one.

\begin{figure*}[t]
\begin{center}
\includegraphics[scale=1]{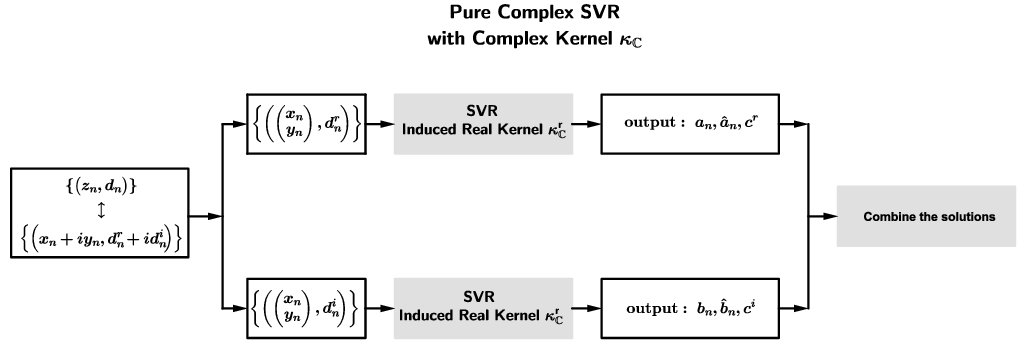}
\end{center}
\caption{Pure complex SVR. The difference with the dual channel approach is due to the incorporation of the induced real kernel $\kappa^r_{\C}$, which depends on the selection of the complex kernel $\kappa_{\C}$. In this context one exploits the complex structure of the space, which is lost in the dual channel approach.}\label{FIG:sketch}
\end{figure*}

Observe that at the training points, i.e., $\Phi_{\C}(\bz_n)$, $T$ takes the values $T(\Phi_{\C}(\bz_n))$.  Following similar arguments as  with the real case, this is equivalent with finding a complex non-linear function $g$ defined on $\cX$ such that
\begin{align}\label{EQ:complex_estimation}
g(\bz) = T\circ\Phi_{\C}(\bz) = \langle \Phi_{\C}(\bz), w\rangle_{\HH} + \langle \Phi_{\C}^*(\bz), v \rangle_{\HH} + c,
\end{align}
for some $w,v\in\HH$, $c\in\C$, which satisfies the aforementioned properties. We formulate the complex support vector regression task as follows:
{\small
\begin{align}\label{EQ:complex_primal}
\begin{matrix}
 \displaystyle{\min_{w,v,c}} & \begin{matrix}\frac{1}{2}\|w\|^2_{\HH} + \frac{1}{2}\|v\|^2_{\HH}
                                                + \frac{C}{N}\displaystyle{\sum_{n=1}^N} (\xi^r_n + \hat\xi^r_n + \xi^i_n + \hat\xi^i_n)\cr
                                                \end{matrix}
                                                \cr
\textrm{s. t.} &
\left\{\begin{matrix}   \real(\langle \Phi_{\C}(\bz_n), w\rangle_{\HH} + \langle \Phi_{\C}(\bz_n), v \rangle_{\HH}+ c - d_n) & \leq & \epsilon + \xi^r_n\cr
                        \real(d_n - \langle \Phi_{\C}(\bz_n), w\rangle_{\HH} -\langle \Phi_{\C}^*(\bz_n), v\rangle_{\HH} - c) & \leq & \epsilon + \hat\xi^r_n \cr
                        \imag(\langle \Phi_{\C}(\bz_n), w\rangle_{\HH} + \langle \Phi_{\C}(\bz_n), v\rangle_{\HH}+ c - d_n) & \leq & \epsilon + \xi^i_n\cr
                        \imag(d_n - \langle \Phi_{\C}(\bz_n), w\rangle_{\HH} -\langle \Phi_{\C}^*(\bz_n), v\rangle_{\HH} - c) & \leq & \epsilon + \hat\xi^i_n \cr
                        \xi^r_n, \hat\xi^r_n, \xi^i_n, \hat\xi^i_n      & \geq & 0
                        \end{matrix}\right.
\end{matrix}
\end{align}
}

To solve (\ref{EQ:complex_primal}), we derive the Lagrangian and the KKT conditions to obtain the dual problem. Thus we take:
{\small
\begin{align}\label{EQ:lagrangian}
\begin{matrix}
\cL = \frac{1}{2}\|w\|^2 + \frac{1}{2}\|v\|^2 + \frac{C}{N}\displaystyle\sum_{n=1}^N (\xi^r_n + \hat\xi^r_n + \xi^i_n + \hat\xi^i_n)\cr
+ \displaystyle\sum_{n=1}^N a_n( \real(\langle \Phi_{\C}(\bz_n), w\rangle_{\HH} + \langle \Phi_{\C}(\bz_n), v \rangle_{\HH} + c - d_n) - \epsilon - \xi^r_n )\cr
+ \displaystyle\sum_{n=1}^N \hat a_n ( \real(d_n - \langle \Phi_{\C}(\bz_n), w\rangle_{\HH} -\langle \Phi_{\C}^*(\bz_n), v\rangle_{\HH} - c) - \epsilon - \hat\xi^r_n )\cr
+ \displaystyle\sum_{n=1}^N b_n ( \imag(\langle \Phi_{\C}(\bz_n), w\rangle_{\HH} + \langle \Phi_{\C}(\bz_n), v\rangle_{\HH} + c - d_n) - \epsilon - \xi^i_n )\cr
+ \displaystyle\sum_{n=1}^N \hat b_n ( \imag(d_n - \langle \Phi_{\C}(\bz_n), w\rangle_{\HH} -\langle \Phi_{\C}^*(\bz_n), v\rangle_{\HH} - c) - \epsilon + \hat\xi^i_n )\cr
- \displaystyle\sum_{n=1}^N \eta_n\xi^r_n - \displaystyle\sum_{n=1}^N \hat\eta_n\hat\xi^r_n - \displaystyle\sum_{n=1}^N \theta_n\xi^i_n - \displaystyle\sum_{n=1}^N \hat\theta_n\hat\xi^i_n,
\end{matrix}
\end{align}
}

\noindent where $a_n$, $\hat a_n$, $b_n$, $\hat b_n$, $\eta_n$, $\hat\eta_n$, $\theta_n$, $\hat\theta_n$ are the Lagrange multipliers. To exploit the saddle point conditions, we employ the rules of Wirtinger calculus for the complex variables on complex RKHS as described in \cite{Bouboulis_2011_10643} and deduce that
{\small
\begin{align*}
\frac{\partial\cL}{\partial w^*} =& \frac{1}{2}w + \frac{1}{2}\displaystyle\sum_{n=1}^N a_n\Phi_{\C}(\bz_n) - \frac{1}{2}\displaystyle\sum_{n=1}^N \hat a_n\Phi_{\C}(\bz_n)\\
& - \frac{\ii}{2}\displaystyle\sum_{n=1}^N b_n\Phi_{\C}(\bz_n) + \frac{\ii}{2}\displaystyle\sum_{n=1}^N \hat b_n\Phi_{\C}(\bz_n),
\end{align*}

\begin{align*}
\frac{\partial\cL}{\partial v^*} =& \frac{1}{2}v + \frac{1}{2}\displaystyle\sum_{n=1}^N a_n\Phi_{\C}^*(\bz_n) - \frac{1}{2}\displaystyle\sum_{n=1}^N \hat a_n\Phi_{\C}^*(\bz_n)\\
& - \frac{\ii}{2}\displaystyle\sum_{n=1}^N b_n\Phi_{\C}^*(\bz_n) + \frac{\ii}{2}\displaystyle\sum_{n=1}^N \hat b_n\Phi_{\C}^*(\bz_n),
\end{align*}

\begin{align*}
\frac{\partial\cL}{\partial c^*} =& \frac{1}{2}\displaystyle\sum_{n=1}^N a_n - \frac{1}{2}\displaystyle\sum_{n=1}^N \hat a_n
+ \frac{\ii}{2}\displaystyle\sum_{n=1}^N b_n - \frac{\ii}{2}\displaystyle\sum_{n=1}^N \hat b_n.
\end{align*}
}

\noindent For the real variables we compute the gradients in the traditional way:
\begin{align*}
\begin{matrix}
\frac{\partial\cL}{\partial \xi^r_n} = \frac{C}{N} - a_n - \eta_n,& \frac{\partial\cL}{\partial \hat\xi^r_n} = \frac{C}{N} - \hat a_n - \hat\eta_n,\cr
\frac{\partial\cL}{\partial \xi^i_n} = \frac{C}{N} - b_n - \theta_n, & \frac{\partial\cL}{\partial \hat\xi^i_n} = \frac{C}{N} - \hat b_n - \theta_n.
\end{matrix}
\end{align*}
for all $n=1,\dots,N$.

As all gradients have to vanish for the saddle point conditions, we finally take that
{\small
\begin{align}
w =& \displaystyle\sum_{n=1}^N(\hat a_n - a_n)\Phi_{\C}(\bz_n) - \ii\displaystyle\sum_{n=1}^N(\hat b_n - b_n)\Phi_{\C}(\bz_n),\label{EQ:SP1}\\
v =& \displaystyle\sum_{n=1}^N(\hat a_n - a_n)\Phi_{\C}^*(\bz_n) - \ii\displaystyle\sum_{n=1}^N(\hat b_n - b_n)\Phi_{\C}^*(\bz_n),\label{EQ:SP2}
\end{align}

\begin{align}
\displaystyle\sum_{n=1}^N (\hat a_n - a_n) = \displaystyle\sum_{n=1}^N (\hat b_n - b_n) = 0,\label{EQ:SP3}
\end{align}

\begin{align}
\begin{matrix}
\eta_n = \frac{C}{N} - a_n, & \hat \eta_n = \frac{C}{N} - \hat a_n,\\
\theta_n = \frac{C}{N} - b_n, & \hat \theta_n = \frac{C}{N} - \hat b_n,
\end{matrix}\label{EQ:SP4}
\end{align}
}
for $n=1,\dots,N$.

To compute $\|w\|_{\HH}^2=\langle w, w\rangle_{\HH}$, we apply equation (\ref{EQ:SP1}), Lemma \ref{LEM:imaginary_part_of_complex_kernel}, the reproducing property of $\HH$, i.e., $\langle \Phi(\bz_n), \Phi(\bz_m)\rangle_{\HH} = \kappa_{\C}(\bz_m,\bz_n)$, and the sesqui-linear property of the inner product of $\HH$ to obtain that:
\begin{align*}
\|w\|_{\HH}^2 =& \displaystyle\sum_{n,m=1}^N (\hat a_n - a_n)(\hat a_m - a_m)\kappa^r_{\C}(\bz_m,\bz_n)\\
&+ \displaystyle\sum_{n,m=1}^N (\hat b_n - b_n)(\hat b_m - b_m)\kappa^r_{\C}(\bz_m,\bz_n)\\
& + 2\displaystyle\sum_{n,m=1}^N (\hat a_n - a_n)(\hat b_m - b_m)\kappa^i_{\C}(\bz_m,\bz_n).
\end{align*}
Similarly, we have
\begin{align*}
\|v\|_{\HH}^2 =& \displaystyle\sum_{n,m=1}^N (\hat a_n - a_n)(\hat a_m - a_m)\kappa^r_{\C}(\bz_m,\bz_n)\\
&+ \displaystyle\sum_{n,m=1}^N (\hat b_n - b_n)(\hat b_m - b_m)\kappa^r_{\C}(\bz_m,\bz_n)\\
& - 2\displaystyle\sum_{n,m=1}^N (\hat a_n - a_n)(\hat b_m - b_m)\kappa^i_{\C}(\bz_m,\bz_n),
\end{align*}
and
\begin{align*}
\frac{\langle \Phi_{\C}(\bz_n), w\rangle_{\HH} + \langle \Phi_{\C}^*(\bz_n), v\rangle_{\HH}}{2} &= \displaystyle\sum_{m=1}^N (\hat a_m - a_m) \kappa^r_{\C}(\bz_m,\bz_n)\\
&+ \ii \displaystyle\sum_{m=1}^N (\hat b_m - b_m) \kappa^r_{\C}(\bz_m,\bz_n).
\end{align*}

Eliminating $\eta_n$, $\hat \eta_n$, $\theta_n$, $\hat\theta_n$ via (\ref{EQ:SP4}) and $w, v$ via the aforementioned relations,
we obtain the final form of the Lagrangian:
{\small
\begin{align}
\begin{matrix}
\cL =& -\displaystyle\sum_{n,m=1}^N (\hat a_n - a_n)(\hat a_m - a_m)\kappa^r_{\C}(\bz_m,\bz_n)\\
&-\displaystyle\sum_{n,m=1}^N (\hat b_n - b_n)(\hat b_m - b_m)\kappa^r_{\C}(\bz_m,\bz_n)\\
&-\epsilon \displaystyle\sum_{n=1}^N (a_n+\hat a_n+b_n+\hat b_n)\\
&+\displaystyle \sum_{n=1}^N d^r_n(\hat a_n-a_n) + \displaystyle \sum_{n=1}^Nd^i_n(\hat b_n-b_n),
\end{matrix}
\end{align}
}

\noindent where $d^r_n$, $d^i_n$ are the real and imaginary parts of the output $d_n$, $n=1,\dots,N$.
This means that we can split the dual problem into two separate maximization tasks:
\begin{subequations}
{\small
\begin{align}\label{EQ:complex_dual1}
\begin{matrix}
\displaystyle{\maxim_{\ba, \hat\ba}} & \left\{\begin{matrix}  -\displaystyle{\sum_{n,m=1}^N}(\hat a_n - a_n)(\hat a_m - a_m)\kappa^r_{\C}(\bz_m,\bz_n)\cr
                                                       -\epsilon\displaystyle{\sum_{n=1}^N}(\hat a_n + a_n) + \displaystyle{\sum_{n=1}^N} d^r_n(\hat a_n - a_n) \end{matrix}\right.\cr
\textrm{subject to} & \displaystyle{\sum_{n=1}^N}(\hat a_n - a_n)=0 \textrm{ and } a_n, \hat a_n\in[0,C/N],
\end{matrix}
\end{align}
}

\noindent and

{\small
\begin{align}\label{EQ:complex_dual2}
\begin{matrix}
\displaystyle{\maxim_{\bb, \hat\bb}} & \left\{\begin{matrix}  -\displaystyle{\sum_{n,m=1}^N}(\hat b_n - b_n)(\hat b_m - b_m)\kappa^r_{\C}(\bz_m,\bz_n)\cr
                                                       -\epsilon\displaystyle{\sum_{n=1}^N}(\hat b_n + b_n) + \displaystyle{\sum_{n=1}^N} d^i_n(\hat b_n - b_n) \end{matrix}\right.\cr
\textrm{subject to} & \displaystyle{\sum_{n=1}^N}(\hat b_n - b_n)=0 \textrm{ and } b_n, \hat b_n\in[0,C/N].
\end{matrix}
\end{align}
}
\end{subequations}
Observe that (\ref{EQ:complex_dual1}) and (\ref{EQ:complex_dual2}) are equivalent with the dual problem of a standard real support vector regression task with kernel $2\kappa^r_{\C}$. This is a real kernel, as Lemma \ref{LEM:induced_kernel} establishes. Therefore (Figure \ref{FIG:sketch}), one may solve the two real SVR tasks for $a_n$, $\hat a_n$, $c^r$ and $b_n$, $\hat b_n$, $c^i$, respectively, using any one of the algorithms which have been developed for this purpose, and then combine the two solutions to find the final non-linear solution of the complex problem as
\begin{align}
g(\bz) =& \langle \Phi_{\C}(\bz), w\rangle_{\HH} + \langle \Phi_{\C}^*(\bz), v\rangle_{\HH} + c\nonumber\\
=&\displaystyle 2\sum_{n=1}^N(\hat a_n - a_n)\kappa^r_{\C}(\bz_n,\bz)\label{EQ:solution_complex}\\
&+ 2\ii\sum_{n=1}^N(\hat b_n - b_n)\kappa^r_{\C}(\bz_n,\bz) + c.\nonumber
\end{align}
In this paper we are focusing mainly on the complex Gaussian kernel. It is important to emphasize that, in this case, the induced kernel $\kappa^r_{\C}$ is not the real Gaussian RBF. Figure \ref{FIG:feature_map} shows the element $\kappa^r_{\C}(\cdot,(0,0)^T)$ of the induced real feature space.

\begin{remark}\label{REM:SVR_complexification}
For the complexification procedure, we select a real kernel $\kappa_{\R}$ and transform the input data from $\cX$ to the complexified space $\HH$, via the feature map $\bar\Phi_{\C}$, to obtain the data $\{(\bar\Phi_{\C}(\bz_n), d_n);\;n=1,\dots,N\}$. Following a similar procedure as the one described above and considering that $$\langle\bar\Phi_{\C}(\bz_n), \bar\Phi_{\C}(\bz_m)\rangle_{\HH} = 2\kappa_{\R}(\bz_m, \bz_n)$$
we can easily deduce that the dual of the complexified SVR task is equivalent to two real SVR tasks employing the kernel $2\kappa_{\R}$. Hence, the complexification technique is identical to the DRC approach.
\end{remark}

\section{Complex support vector machine}\label{SEC:CSVM}

\subsection{Complex hyperplanes}
Recall that in any real Hilbert space $\cH$, a hyperplane consists of all the elements $f\in\cH$ that satisfy
\begin{align}\label{EQ:hyperplane_real}
\langle f, w\rangle_{\cH} + c =0,
\end{align}
for some $w\in\cH$, $c\in\R$. Moreover, as Figure \ref{FIG:hyperplane_R} shows, any hyperplane of $\cH$ divides the space into two parts, $\cH_+ = \{f\in\cH;\; \langle f, w\rangle_{\cH} + c > 0\}$ and $\cH_- = \{f\in\cH;\; \langle f, w\rangle_{\cH} + c < 0\}$. In the traditional SVM classification task the goal is to separate two distinct classes of data by a maximum margin hyperplane, so that one class falls into $\cH_+$ and the other into $\cH_-$ (excluding some outliers). In order to be able to generalize the SVM rationale to complex spaces, firstly, we need to determine an appropriate definition for a complex hyperplane. The difficulty is that the set of complex numbers is not an ordered one, and thus one may not assume that a complex version of (\ref{EQ:hyperplane_real}) divides the space into two parts, as $\cH_+$ and $\cH_-$ cannot be defined. Instead, we will provide a novel definition of complex hyperplanes that divide the complex space into four parts. This will be our kick off point for deriving the complex SVM rationale, which classifies objects into four (instead of two) classes.

\begin{figure}[t]
\begin{center}
\includegraphics[scale=1.5]{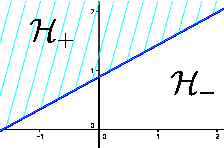}
\end{center}
\caption{A hyperplane separates the space $\cH$ into two parts, $\cH_+$ and $\cH_-$.}\label{FIG:hyperplane_R}
\end{figure}

\begin{lemma}\label{LEM:hyperplane_complex}
The relations
\begin{subequations}\label{EQ:hyperplane_complex}
\begin{align}
\real\left(\langle f, w\rangle_{\HH} + c\right) =0, \label{EQ:hyperplane_complex1}\\
\imag\left(\langle f, w\rangle_{\HH} + c\right) =0, \label{EQ:hyperplane_complex2}
\end{align}
\end{subequations}
for some $w\in\HH$, $c\in\C$, where $f\in\HH$, represent two orthogonal hyperplanes of the doubled real space, i.e., $\cH^2$, in general positions.
\end{lemma}
\begin{proof}
Observe that
\begin{align*}
\langle f, w\rangle_{\HH} &= \langle f^r, w^r\rangle_{\cH} + \langle f^i, w^i\rangle_{\cH} +\ii (\langle f^i, w^r\rangle_{\cH} - \langle f^r, w^i\rangle_{\cH}),
\end{align*}
where $f=f^r+\ii f^i$, $w=w^r+\ii w^i$.
Hence, we take that
\begin{align*}
\left\langle \left(\begin{matrix}f^r \cr f^i\end{matrix}\right), \left(\begin{matrix}w^r \cr w^i\end{matrix}\right)\right\rangle_{\cH^2} + c^r = 0
\end{align*}
and
\begin{align*}
\left\langle \left(\begin{matrix}f^r \cr f^i\end{matrix}\right), \left(\begin{matrix}-w^i \cr w^r\end{matrix}\right)\right\rangle_{\cH^2} + c^i = 0,
\end{align*}
where $c=c^r+\ii c^i$. These are two distinct hyperplanes of $\cH^2$.
Moreover, as
\begin{align*}
\left(\begin{matrix}-w^i & w^r\end{matrix}\right) \left(\begin{matrix}w^r \cr w^i\end{matrix}\right) = 0,
\end{align*}
the two hyperplanes are orthogonal.
\end{proof}

\begin{lemma}\label{LEM:hyperplane_wl_complex}
The relations
\begin{subequations} \label{EQ:hyperplane_wl_complex}
\begin{align}
\real\left(\langle f, w\rangle_{\HH} + \langle f^*, v\rangle_{\HH} + c\right) =0, \label{EQ:hyperplane_wl_complex1}\\
\imag\left(\langle f, w\rangle_{\HH} + \langle f^*, v\rangle_{\HH} + c\right) =0, \label{EQ:hyperplane_wl_complex2}
\end{align}
\end{subequations}
for some $w,v\in\HH$, $c\in\C$, where $f\in\HH$, represent two hyperplanes of the doubled real space, i.e., $\cH^2$. Depending on the values of $w,v$, these hyperplanes may be placed arbitrarily on $\cH^2$.
\end{lemma}
\begin{proof}
Following a similar rationale as in the proof of Lemma \ref{LEM:hyperplane_complex}, we take
\begin{align*}
\left\langle \left(\begin{matrix}f^r \cr f^i\end{matrix}\right), \left(\begin{matrix}w^r + v^r \cr w^i-v^i\end{matrix}\right)\right\rangle_{\cH^2} + c^r = 0
\end{align*}
and
\begin{align*}
\left\langle \left(\begin{matrix}f^r \cr f^i\end{matrix}\right), \left(\begin{matrix}-(w^i +v^i) \cr w^r-v^r\end{matrix}\right)\right\rangle_{\cH^2} + c^i = 0,
\end{align*}
where $f=f^r+\ii f^i$, $w=w^r+\ii w^i$, $v=v^r+\ii v^i$, $c=c^r+\ii c^i$.
\end{proof}

The following Definition comes naturally.

\begin{definition}\label{DEF:hyperplanes_wl}
Let $\HH$ be a complex Hilbert space. We define the complex pair of hyperplanes as the set of all $f\in\HH$ that satisfy one of the following relations
\begin{subequations}\label{EQ:hyperplanes_wl}
\begin{align}
\real\left(\langle f, w\rangle_{\HH} + \langle f^*, v\rangle_{\HH} + c\right) =0, \\
\imag\left(\langle f, w\rangle_{\HH} + \langle f^*, v\rangle_{\HH} + c\right) =0,
\end{align}
\end{subequations}
for some $w,v\in\HH$, $c\in\C$.
\end{definition}

Lemmas \ref{LEM:hyperplane_complex} and \ref{LEM:hyperplane_wl_complex} demonstrate the significant difference between complex linear estimation and widely linear estimation functions, which has been, already, pointed out in Section \ref{SEC:complex_SVR}, albeit in a different context. The complex linear case is quite restrictive, as the couple of complex hyperplanes are always orthogonal. On the other hand, the widely linear case is more general and covers all cases. The complex pair of hyperplanes (as defined by definition  \ref{DEF:hyperplanes_wl}) divides the space into four parts, i.e.,
\begin{align*}
\cH_{++} = \left\{f\in\cH;\; \begin{matrix}\real\left(\langle f, w\rangle_{\HH} + \langle f^*, v\rangle_{\HH} + c\right) > 0, \cr \imag\left(\langle f, w\rangle_{\HH} + \langle f^*, v\rangle_{\HH} + c\right) > 0\end{matrix}\right\},\\
\cH_{+-} = \left\{f\in\cH;\; \begin{matrix}\real\left(\langle f, w\rangle_{\HH} + \langle f^*, v\rangle_{\HH} + c\right) > 0, \cr \imag\left(\langle f, w\rangle_{\HH} + \langle f^*, v\rangle_{\HH} + c\right) < 0\end{matrix}\right\},\\
\cH_{-+} = \left\{f\in\cH;\; \begin{matrix}\real\left(\langle f, w\rangle_{\HH} + \langle f^*, v\rangle_{\HH} + c\right) < 0, \cr \imag\left(\langle f, w\rangle_{\HH} + \langle f^*, v\rangle_{\HH} + c\right) > 0\end{matrix}\right\},\\
\cH_{--} = \left\{f\in\cH;\; \begin{matrix}\real\left(\langle f, w\rangle_{\HH} + \langle f^*, v\rangle_{\HH} + c\right) < 0, \cr \imag\left(\langle f, w\rangle_{\HH} + \langle f^*, v\rangle_{\HH} + c\right) < 0\end{matrix}\right\}.
\end{align*}
Figure \ref{FIG:hyperplane_C} demonstrates a simple case of a complex pair of hyperplanes that divides $\C$ into four parts.
Note that in some cases the complex pair of hyperplanes might degenerate into two identical hyperplanes or two parallel hyperplanes.

\begin{figure}[t]
\begin{center}
\includegraphics[scale=0.07]{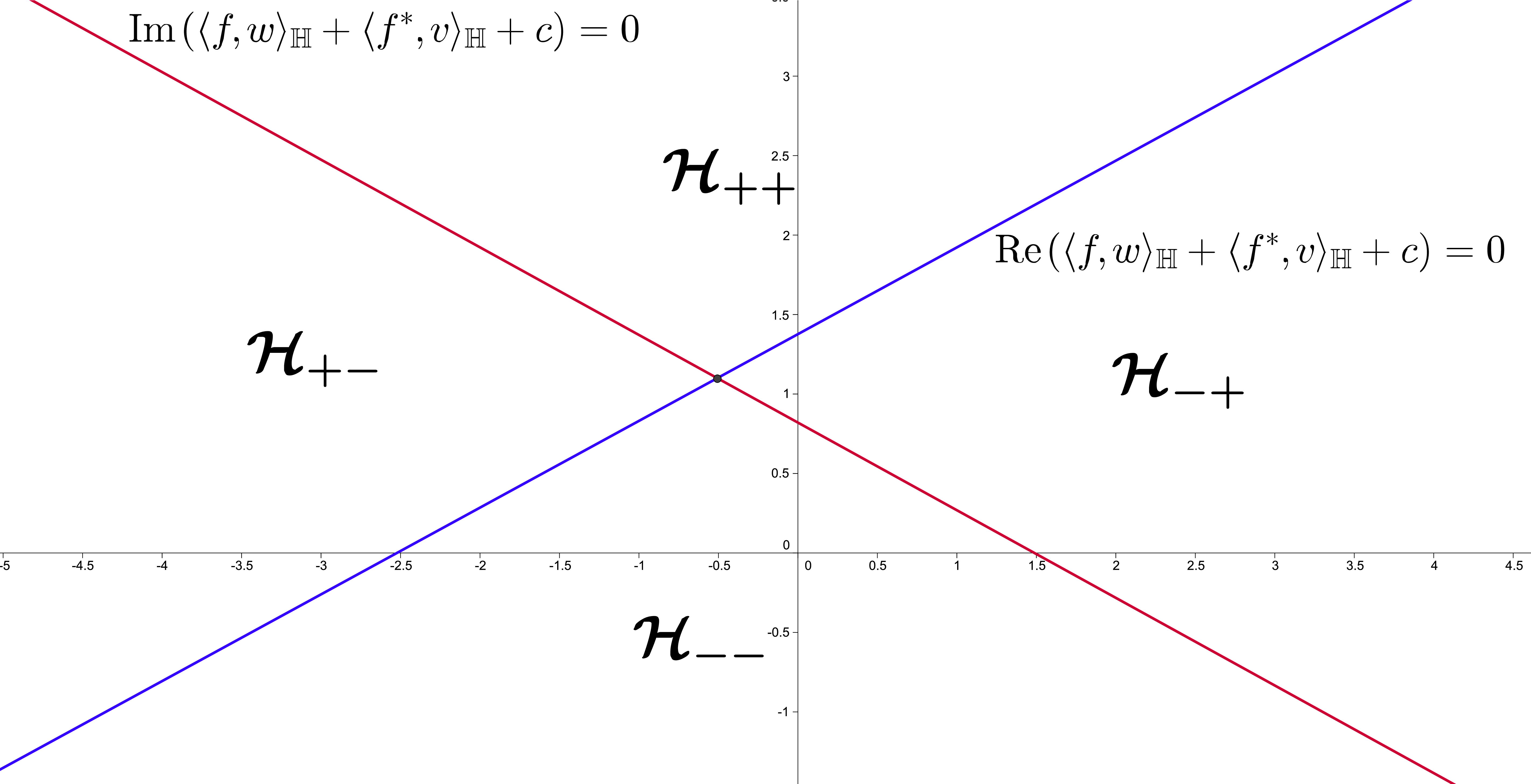}
\end{center}
\caption{A complex pair of hyperplanes separates the space of complex numbers (i.e., $\HH=\C$) into four parts.}\label{FIG:hyperplane_C}
\end{figure}

\subsection{The quaternary complex SVM}
The complex SVM classification task can be formulated as follows. Suppose we are given training data, which belong to four separate classes $C_{++}, C_{+-}, C_{-+}, C_{--}$, i.e., $\{(\bz_n, d_n);\;n=1,\dots,N\}\subset\cX\times\{\pm 1 \pm \ii)\}$. If $d_n=+1 + \ii$, then the $n$-th sample belongs to $C_{++}$, i.e., $\bz_n\in C_{++}$, if $d_n=1 - \ii$, then $\bz_n\in C_{+-}$, if $d_n=-1 + \ii$, then $\bz_n\in C_{-+}$ and if $d_n=-1 - \ii$, then $\bz_n\in C_{--}$. Consider the complex RKHS $\HH$ with respective kernel $\kappa_{\C}$. Following a similar rationale to the real case, we transform the input data from $\cX$ to $\HH$, via the feature map $\Phi_{\C}$. The goal of the SVM task is to estimate a complex pair of maximum margin hyperplanes that separates the points of the four classes (see Figure \ref{FIG:classes4}). Thus, we need to minimize

\begin{flalign*}
\left\|\left(\begin{matrix}w^r + v^r \cr w^i-v^i\end{matrix}\right)\right\|_{\cH^2}^2 + \left\|\left(\begin{matrix}-(w^i +v^i) \cr w^r-v^r\end{matrix}\right)\right\|_{\cH^2}^2 &= \\
\|w^r + v^r\|^2_{\cH} + \|w^i-v^i\|^2_{\cH} + \|(w^i +v^i)\|^2_{\cH} + \|w^r-v^r\|^2_{\cH} &= \\
2\|w^r\|_{\cH}^2 + 2\|w^i\|_{\cH}^2 + 2\|v^r\|_{\cH}^2 + 2\|v^i\|_{\cH}^2 &= \\
2(\|w\|_{\HH}^2 + \|v\|^2_{\HH})&.
\end{flalign*}

\begin{figure}[t]
\begin{center}
\includegraphics[scale=0.4]{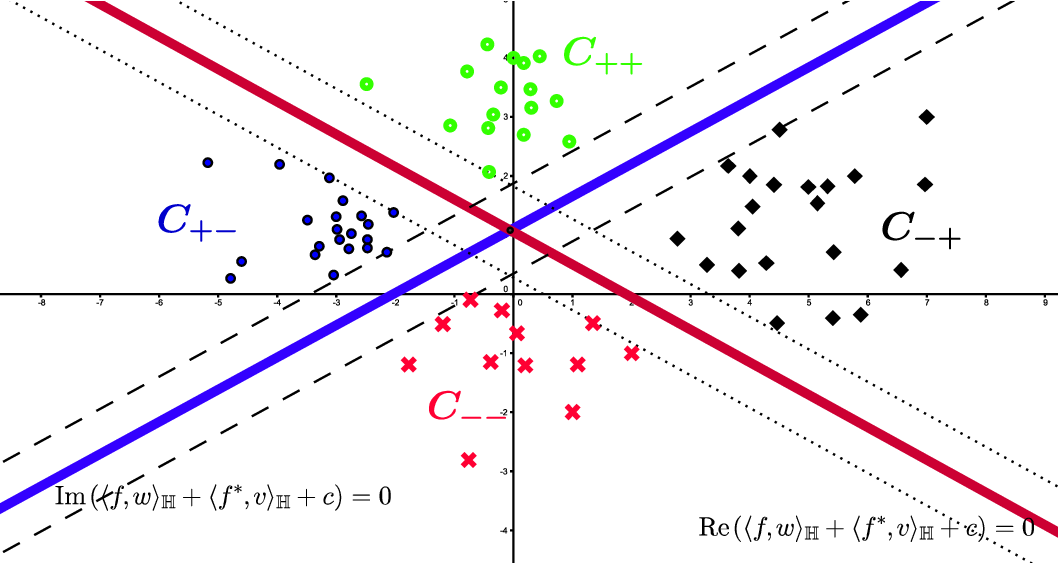}
\end{center}
\caption{A complex pair of hyperplanes that separates the four given classes. The hyperplanes are chosen so that to maximize the margin between the classes.}\label{FIG:classes4}
\end{figure}

Therefore, the primal complex SVM optimization problem can be formulated as
{\small
\begin{align}\label{EQ:SVM_complex_primal}
\begin{matrix}
\displaystyle{\min_{w,v,c}} & \frac{1}{2}\|w\|^2_{\HH} + \frac{1}{2}\|v\|^2_{\HH} + \frac{C}{N}\displaystyle{\sum_{n=1}^N}(\xi_n^r + \xi_n^i)\cr
\textrm{s. to} &
\left\{\begin{matrix} d^r_n\real\left(\langle \Phi_{\C}(\bz_n), w\rangle_{\HH} + \langle \Phi^*_{\C}(\bz_n), v\rangle_{\HH} + c\right) & \geq & 1 - \xi_n^r\cr
                      d^i_n\imag\left(\langle \Phi_{\C}(\bz_n), w\rangle_{\HH} + \langle \Phi^*_{\C}(\bz_n), w\rangle_{\HH} + c\right) & \geq & 1 - \xi_n^i\cr
                      \xi_n^r,\xi_n^i                               & \geq & 0
                      \end{matrix}\right.\cr
& \textrm{for } n=1,\dots,N.
\end{matrix}
\end{align}
}
The Lagrangian function becomes
{\small
\begin{align*}
L&(w,v,\ba,\bb) = \frac{1}{2}\|w\|^2_{\HH} + \frac{1}{2}\|v\|^2_{\HH} + \frac{C}{N}\displaystyle{\sum_{n=1}^N}(\xi_n^r + \xi_n^i)\\
& -\displaystyle{\sum_{n=1}^N} a_n \left(d^r_n\real\left(\langle \Phi_{\C}(\bz_n), w\rangle_{\HH} + \langle \Phi^*_{\C}(\bz_n), v\rangle_{\HH} + c\right) - 1 + \xi_n^r \right)\\
& - \displaystyle{\sum_{n=1}^N} b_n \left(d^i_n\imag\left(\langle \Phi_{\C}(\bz_n), w\rangle_{\HH} + \langle \Phi^*_{\C}(\bz_n), w\rangle_{\HH} + c\right) - 1 + \xi_n^i\right)\\
&- \displaystyle{\sum_{n=1}^N} \eta_n\xi_n^r - \displaystyle{\sum_{n=1}^N}\theta_n\xi_n^i,
\end{align*}
}

\noindent where $a_n, b_n,\eta_n$ and $\theta_n$ are the positive Lagrange multipliers of the respective inequalities, for $n=1,\dots,N$.
To exploit the saddle point conditions of the Lagrangian function, we employ the rules of Wirtinger calculus to compute the respective gradients. Hence, we take
\begin{align*}
\frac{\partial L}{\partial w^*} &= \frac{1}{2}w - \frac{1}{2}\displaystyle\sum_{n=1}^N a_n d_n^r\Phi_{\C}(\bz_n) + \frac{\ii}{2}\displaystyle\sum_{n=1}^N b_n d_n^i\Phi_{\C}(\bz_n)\\
\frac{\partial L}{\partial v^*} &= \frac{1}{2}v - \frac{1}{2}\displaystyle\sum_{n=1}^N a_n d_n^r\Phi^*_{\C}(\bz_n) + \frac{\ii}{2}\displaystyle\sum_{n=1}^N b_n d_n^i\Phi^*_{\C}(\bz_n)\\
\frac{\partial L}{\partial c^*} &= \frac{1}{2}\displaystyle\sum_{n=1}^N a_n d_n^r + \frac{\ii}{2}\displaystyle\sum_{n=1}^N b_n d_n^i
\end{align*}
and
\begin{align*}
\frac{\partial L}{\partial \xi^r_n} = \frac{C}{N} - a_n - \eta_n,\quad \frac{\partial L}{\partial \xi_n^i} = \frac{C}{N} - b_n - \theta_n.
\end{align*}
for $n=1,\dots,N$. As all the gradients have to vanish, we finally take that
\begin{align*}
w &= \displaystyle\sum_{n=1}^N(a_n d_n^r - \ii b_n d_n^i)\Phi_{\C}(\bz_n)\\
v &= \displaystyle\sum_{n=1}^N(a_n d_n^r - \ii b_n d_n^i)\Phi^*_{\C}(\bz_n)\\
\displaystyle\sum_{n=1}^N a_n d_n^r &= \displaystyle\sum_{n=1}^N b_n d_n^i= 0
\end{align*}
and
\begin{align*}
a_n + \eta_n = \frac{C}{N},\quad b_n + \theta_n = \frac{C}{N}
\end{align*}
for $n=1,\dots,N$. Following a similar procedure as in the complex SVR case, it turns out that the dual problem can be split into two separate maximization tasks:
\begin{subequations}
\begin{align}\label{EQ:SVM_complex_dual1}
\begin{matrix}
\displaystyle{\maxim_{\ba}} & \displaystyle{\sum_{n=1}^N}a_n - \displaystyle{\sum_{n,m=1}^N}a_n a_m d_n^r d_m^r\kappa_{\C}^r(\bz_m,\bz_n) \cr
\textrm{subject to} &
\left\{\begin{matrix} \displaystyle\sum_{n=1}^N a_n d_n^r=0 \cr
                      0\leq a_n\leq\frac{C}{N}
                      \end{matrix}\right.\cr
& \textrm{for } n=1,\dots,N
\end{matrix}
\end{align}
and
\begin{align}\label{EQ:SVM_complex_dual2}
\begin{matrix}
\displaystyle{\maxim_{\bb}} & \displaystyle{\sum_{n=1}^N}b_n - \displaystyle{\sum_{n,m=1}^N}b_n b_m d_n^i d_m^i\kappa_{\C}^r(\bz_m,\bz_n) \cr
\textrm{subject to} &
\left\{\begin{matrix} \displaystyle\sum_{n=1}^N b_n d_n^i=0 \cr
                      0\leq b_n\leq\frac{C}{N}
                      \end{matrix}\right.\cr
& \textrm{for } n=1,\dots,N.
\end{matrix}
\end{align}
\end{subequations}

Observe that, similar to the regression case, these problems are equivalent with two distinct real SVM (dual) tasks employing the induced real kernel $2\kappa_{\C}^r$. One may split the (output) data to their real and imaginary parts, as Figure \ref{FIG:sketch_CSVM} demonstrates, solve two real SVM tasks employing any one of the standard algorithms and, finally, combine the solutions to take the complex labeling function:
\begin{align*}
g(\bz) =& \sign_{\ii}\left(\langle \Phi_{\C}(\bz), w\rangle_{\HH} + \langle \Phi^*_{\C}(\bz), v\rangle_{\HH} + c\right)\\
=& \sign_{\ii}\left(2\displaystyle\sum_{n=1}^N(a_n d_n^r+\ii b_n d_n^i)\kappa_{\C}^r(\bz_n,\bz) + c^r + \ii c^i\right),
\end{align*}
where $\sign_{\ii}(z) = \sign(\real(z)) + \ii\sign(\imag(z))$.

\begin{figure*}[t]
\begin{center}
\includegraphics[scale=1]{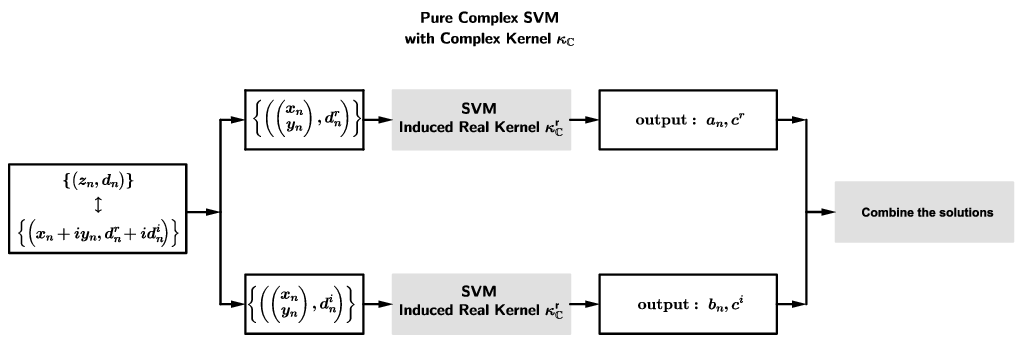}
\end{center}
\caption{Pure Complex support vector machine.}\label{FIG:sketch_CSVM}
\end{figure*}

\begin{remark}\label{REM:SVM_complexification}
Following the complexification procedure, as in Remark \ref{REM:SVR_complexification}, we select a real kernel $\kappa_{\R}$ and transform the input data from $\cX$ to the complexified space $\HH$, via the feature map $\bar\Phi_{\C}$. We can easily deduce that the dual of the complexified SVM task is equivalent to two real SVM tasks employing the kernel $2\kappa_{\R}$.
\end{remark}

\begin{remark}\label{REM:SVM_real_case}
It is evident that both the complex and the complexified SVM can be employed for binary classification as well. The advantage in this case is that one is able to handle complex input data in both scenarios. Moreover, the popular one-versus-one and one-versus-all strategies \cite{Scholkopf_2002_2276, ShaweTaylor_2004_9255}, which address multiclassification problems, can be directly applied to complex inputs using either the complex or the complexified binary SVM.
\end{remark}

\section{Experiments}\label{SEC:EXP}
In order to illuminate the advantages that are gained by the complex kernels and to demonstrate the performance of the proposed algorithmic schemes, we compare it with standard real-valued technics and the dual real channel approach, under various regression and classification scenarios.  In the following, we will refer to the pure complex kernel rationale and the complexification trick, presented in this paper, using the terms CSVR (or CSVM) and complexified SVR (or complexified SVM) respectively. The dual real channel approach, outlined in Section \ref{SEC:DRC}, will be denoted as DRC-SVR. Recall that the DRC approach is equivalent to the complexified rationale, although the latter often provides for more compact formulas and simpler representations. The following experiments were implemented in Matlab. The respective code can be found in \url{bouboulis.mysch.gr/kernels.html}.

\subsection{SVR - Function Estimation}
In this Section, we perform a simple regression test on the complex function $\sinc(z)$. An orthogonal grid of $33\times 9$ actual points of the $sinc$ function, corrupted by noise, was adopted as the training data. Figures \ref{FIG:sinc_real_noisy} and \ref{FIG:sinc_imag_noisy} show the real and imaginary parts of the reconstructed function using the CSVR rationale. Note the excellent visual results obtained by the corrupted training data. Figures \ref{FIG:error_complex}, \ref{FIG:error_DC} and Table \ref{TAB:ESTIMATION} compare the square errors (i.e. $|\hat d_n - \sinc(\bz_n)|^2$, where $\hat d_n$ is the value of the estimated function at $\bz_n$) between the CSVR and the DRC-SVR over 100 different realizations of the experiment. In each realization, the $\sinc$ function was corrupted by white noise of the form $Z=X+\ii Y$, where $X$ and $Y$ are real random variables following the Gaussian distribution with variances $\sigma_1=0.4$ and $\sigma_2=0.3$ respectively. As it is shown in Table \ref{TAB:ESTIMATION}, the DRC-SVR fails to capture the complex structure of the function. On the other hand, the CSVR rationale provides for an estimation function, which exhibits excellent characteristics. A closer look at Figures \ref{FIG:error_complex} and \ref{FIG:error_DC} reveals that at the border of the training grid the square error increases in some cases. This is expected, as the available information (i.e., the neighboring points), which it is exploited by the SVR algorithm, is reduced in these areas compared to the interior points of the grid. Besides the significant decrease in the square error, in these experiments we also observed a significant reduction in the computing time needed for the CSVR compared to the DRC-SVR. In our opinion, this enhanced performance (both in terms of MSE and computational time) is due to the special structure of the $\sinc$ function. Recall that the $\sinc$ is a complex analytic function, hence it is more natural to use complex analytic functions (e.g., the complex Gaussian kernel function), instead of real analytic functions (e.g., the real Gaussian kernel), to estimate its shape.  Both algorithms were implemented in MatLab on a computer with a Core i5 650 microprocessor running at 3.2 GHz.

\begin{table}
\begin{center}
\begin{tabular}{|c|c|c|}
\hline
 & CSVR & DRC-SVR\\\hline
Mean MSE (dB) & -15.75 dB & -10.42 dB\\\hline
Mean number of Support Vectors & 282 & 282\\\hline
Mean Time & 179 secs & 430 secs\\\hline
\end{tabular}
\end{center}
\caption{The mean square errors, the number of support vectors and the computing time over 100 realizations of the $\sinc$ estimation experiment.}\label{TAB:ESTIMATION}
\end{table}

In all the performed experiments, the SMO algorithm was employed using the complex Gaussian kernel and the real Gaussian kernel for the CSVR and the DRC-SVR, respectively (see \cite{Platt_1998_9325}). The parameters of the kernel for both the complex SVR and the DRC-SVR tasks were tuned (using cross-correlation) to provide the smallest mean square error. In particular for the CSVR, the parameter of the complex Gaussian kernel was set to $t=0.3$, while for the DRC-SVR the parameter was set to $t=2$. In both cases the parameters of the SVR task were set as $C=1000$, $\epsilon=0.1$.

%\begin{figure*}
%\centering
%\begin{minipage}[b]{.4\textwidth}
%\begin{center}
%\includegraphics[scale=0.5]{./sinc_re.eps}
%\end{center}
%\caption{The real part ($\real(\sinc(z))$) of the estimated $\sinc$ function from the complex SVR. The points shown in the Figure are the real parts of the training data used in the simulation.}\label{FIG:sinc_real}
%\end{minipage}\qquad
%\begin{minipage}[b]{.4\textwidth}
%\begin{center}
%\includegraphics[scale=0.5]{./sinc_im.eps}
%\end{center}
%\caption{The imaginary part ($\imag(\sinc(z))$) of the estimated $\sinc$ function from the complex SVR. The points shown in the Figure are the imaginary parts of the training data used in the simulation.}\label{FIG:sinc_imag}
%\end{minipage}
%\end{figure*}

\begin{figure*}
\centering
\begin{minipage}[b]{.4\textwidth}
\begin{center}
\includegraphics[scale=0.4]{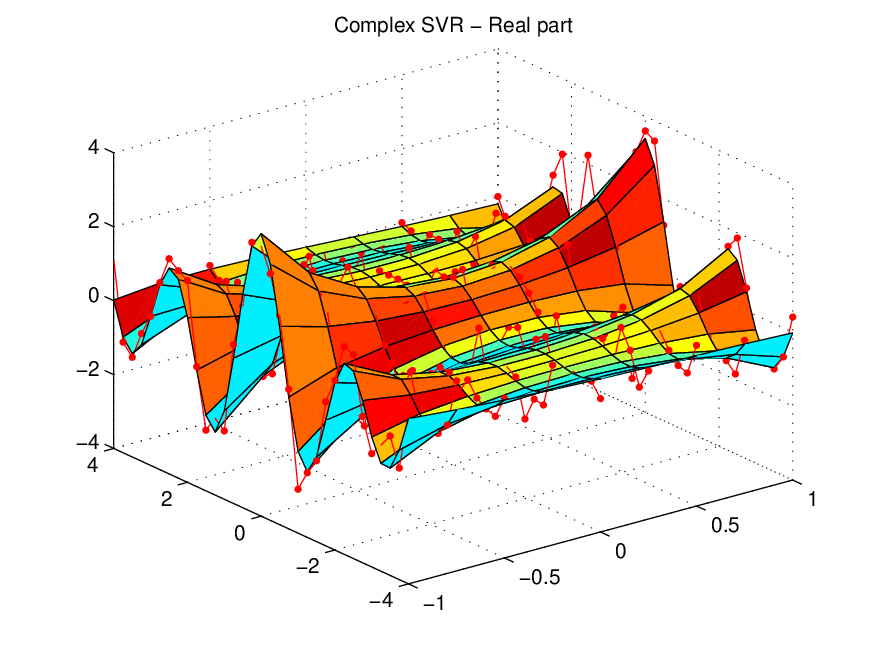}
\end{center}
\caption{The real part ($\real(\sinc(z))$) of the estimated $sinc$ function from the complex SVR. The points shown in the Figure are the real parts of the noisy training data used in the simulation.}\label{FIG:sinc_real_noisy}
\end{minipage}\qquad
\begin{minipage}[b]{.4\textwidth}
\begin{center}
\includegraphics[scale=0.4]{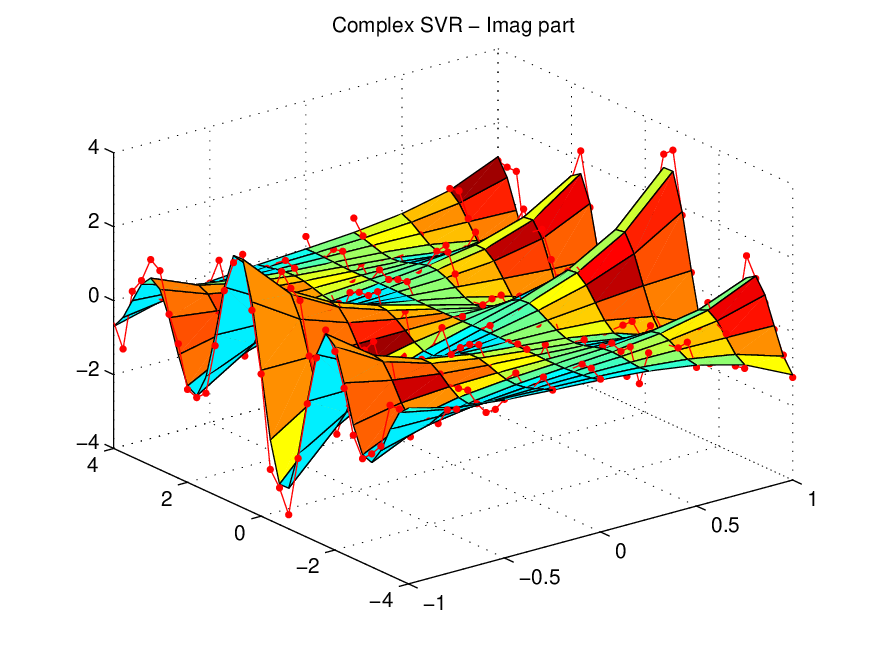}
\end{center}
\caption{The imaginary part ($\imag(\sinc(z))$) of the estimated $sinc$ function from the complex SVR. The points shown in the Figure are the imaginary parts of the noisy training data used in the simulation.}\label{FIG:sinc_imag_noisy}
\end{minipage}
\end{figure*}

\begin{figure*}
\centering
\begin{minipage}[b]{.4\textwidth}
\begin{center}
\includegraphics[scale=0.5]{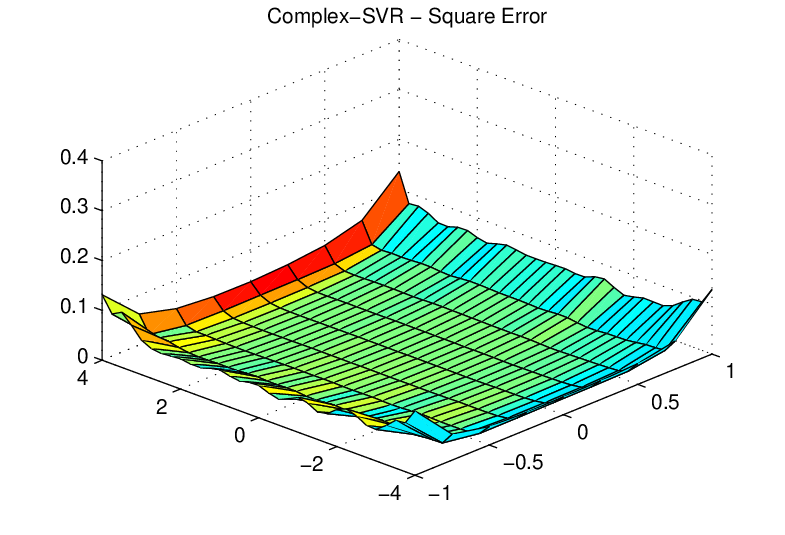}
\end{center}
\caption{The mean square errors (the actual values of the function minus the estimated ones) over 100 realizations of the complex SVR of the $\sinc$ function. The mean square error of all the estimated values over the complete set of 100 experiments was equal to $0.054$ ($-15.75 dB$). }\label{FIG:error_complex}
\end{minipage}\qquad
\begin{minipage}[b]{.4\textwidth}
\begin{center}
\includegraphics[scale=0.45]{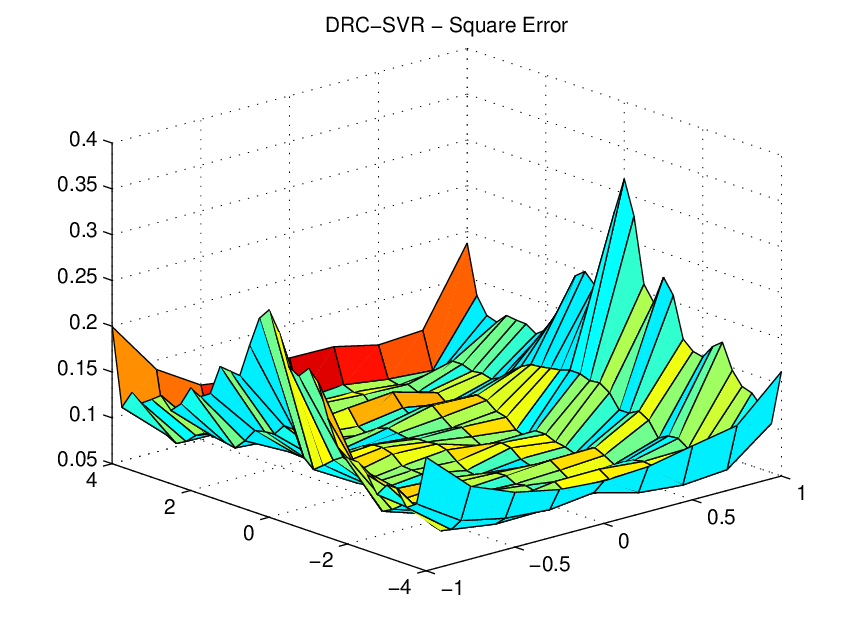}
\end{center}
\caption{The mean square errors (the actual values of the function minus the estimated ones) over 100 realizations of the DRC regression of the $\sinc$ function. The mean square error of all the estimated values over the complete set of 100 experiments was equal to $0.102$ ($-10.42 dB$).}\label{FIG:error_DC}
\end{minipage}
\end{figure*}

\subsection{SVR - Channel Identification}
In this Section, we consider a non-linear channel identification task (see \cite{Adali_2010_10640}). This channel consists of the 5-tap linear component:
\begin{align}\label{EQ:linear_filter}
t(n) = \sum_{k=1}^5 h(k)\cdot s(n-k+1),
\end{align}
where
{\small
\begin{align*}
h(k)=0.432\left(1+\cos\left(\frac{2\pi(k-3)}{5}\right)-\left(1+\cos\frac{2\pi(k-3)}{10}\right)i\right),
\end{align*}
}
for $k=1,\dots, 5$, and the nonlinear component:
\begin{align*}
x(n) = t(n) + (0.15-0.1\ii)t^2(n).
\end{align*}
This is a standard model that has been extensively used in the literature for such tasks, e.g., \cite{SebBuck, SVM_multiregression, Bouboulis_2011_10643, Boub_ACKLMS, Bouboulis_2012_1}.
At the receiver's end, the signal is corrupted by white Gaussian noise and then observed as $y_n$. The signal-to-noise ratio (SNR) was set to 15 dB.
The input signal that was fed to the channel had the form
\begin{align}\label{EQ:ident_input}
s(n) = \left(\sqrt{1-\rho^2}X(n) + \ii\rho Y(n)\right),
\end{align}
where $X(n)$ and $Y(n)$ are Gaussian random variables.
This input is circular for $\rho=\sqrt{2}/2$ and highly non-circular if $\rho$ approaches 0 or 1 \cite{Adali_2010_10640}.
The CSVR and the DRC-SVR rationales were used to address the channel identification task, which aims to discover the input-out relationship between $(s(n-L+1), s(n-L+2), \dots, s(n))$ and $y(n)$ (the parameter $L$ was set to $L=5$). In each experiment, a set of 150 pairs of samples was used to perform the training.  After training, a set of 600 pairs of samples was used to test the estimation's performance of both algorithms (i.e., to measure the mean square error between the actual channel output, $x(n)$, and the estimated output, $\hat x(n)$). To find the best possible values of the parameters $C$ and $t$, that minimize the mean square error for both SVR tasks, an extensive cross-validation procedure has been employed  (see Tables  \ref{TAB:CSVR_ident}, \ref{TAB:DRCSVR_ident}) in a total of 20 sets of data. Figure \ref{FIG:IDENT} shows the minimum mean square error, which has been obtained for all values of the kernel parameter $t$ versus the SVR parameter $C$ for both cases considering a circular input (see also Figure \ref{FIG:IDENT_SV}). It is evident, that the CSVR approach significantly outperforms the DRC-SVR rationale, both in terms of MSE and computational time (Figure \ref{FIG:IDENT_time}). All the Figures refer to the circular case. As the results for the non-circular case are similar, they are omitted to save space.

\begin{table}
\begin{center}
\begin{tabular}{|c|c|}
\hline
$C$ & t \\\hline
1000 & $1/{6^2}$\\\hline
2000 & $1/{6^2}$\\\hline
5000 &  $1/{8^2}$\\\hline
10000 & $1/{9^2}$\\\hline
20000 & $1/{11^2}$\\\hline
50000 & $1/{13^2}$\\\hline
\end{tabular}
\end{center}
\caption{The values of $C$ and $t$ that minimize the mean square error of the CSVR, for the channel identification task.}\label{TAB:CSVR_ident}
\end{table}

\begin{table}
\begin{center}
\begin{tabular}{|c|c|}
\hline
$C$ & t \\\hline
1000 & $1/{4^2}$\\\hline
2000 & $1/{5^2}$\\\hline
5000 &  $1/{6^2}$\\\hline
10000 & $1/{7^2}$\\\hline
20000 & $1/{7^2}$\\\hline
50000 & $1/{10^2}$\\\hline
\end{tabular}
\end{center}
\caption{The values of $C$ and $t$ that minimize the mean square error of the DRC-SVR, for the channel identification task.}\label{TAB:DRCSVR_ident}
\end{table}

\begin{figure}[t]
\begin{center}
\includegraphics[scale=0.4]{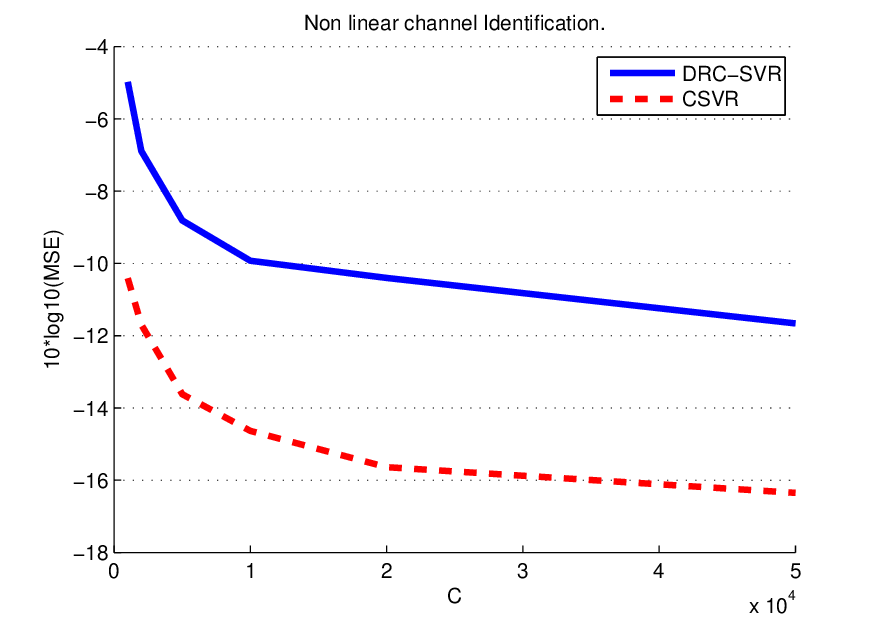}
\end{center}
\caption{MSE versus the SVR parameter $C$ for both the CSVR and the DRC-SVR rationales, for the channel identification task.}\label{FIG:IDENT}
\end{figure}

\begin{figure}[t]
\begin{center}
\includegraphics[scale=0.4]{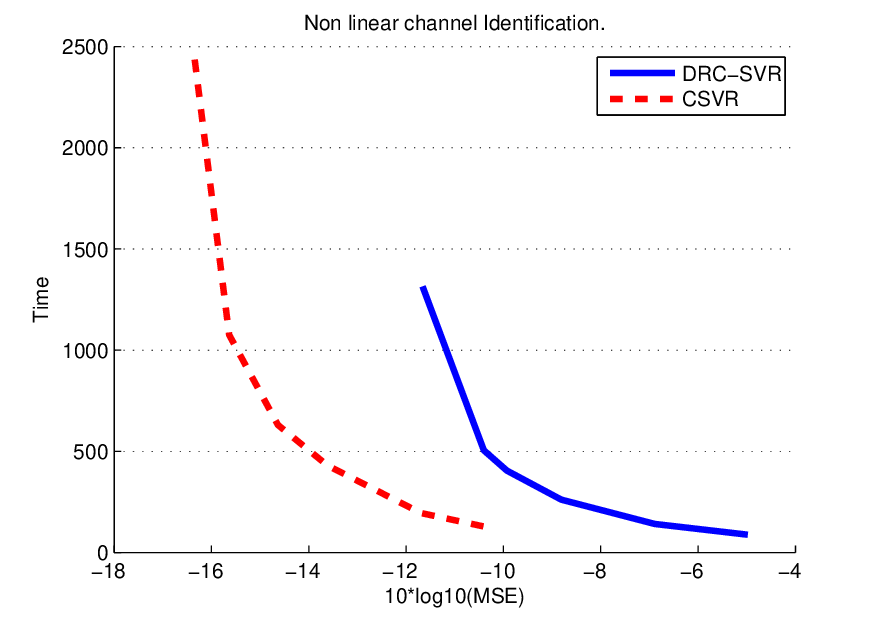}
\end{center}
\caption{Time (in seconds) versus MSE (dB) for both the CSVR and the DRC-SVR rationales, for the channel identification task.}\label{FIG:IDENT_time}
\end{figure}

\begin{figure}[t]
\begin{center}
\includegraphics[scale=0.4]{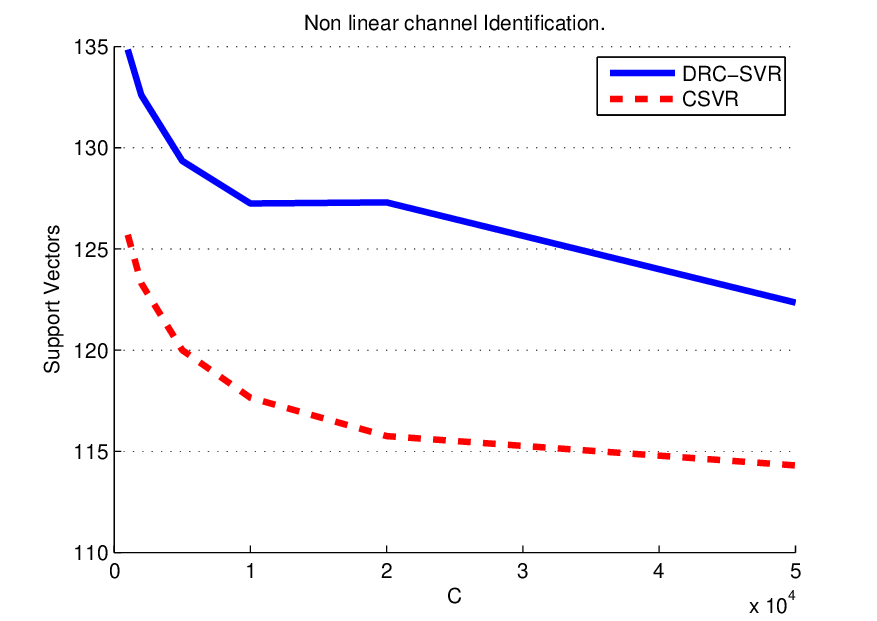}
\end{center}
\caption{The number of Support Vectors versus MSE (dB) for both the CSVR and the DRC-SVR rationales, for the channel identification task.}\label{FIG:IDENT_SV}
\end{figure}

\subsection{SVR - Channel Equalization}
In this Section, we present a non-linear channel equalization task that consists of the linear filter (\ref{EQ:linear_filter})
and the memoryless nonlinearity
\begin {align*}
x(n) =&\; t(n) + (0.1-0.15\ii)\cdot t^2(n)
\end{align*}
At the receiver end of the channel, the signal is corrupted by white Gaussian noise and then observed as $y(n)$. The signal-to-noise ratio was set to 15 dB. The input signal that was fed to the channels had the form
\begin{align}\label{EQ:input}
s(n) = 0.30\left(\sqrt{1-\rho^2}X(n) + \ii\rho Y(n)\right),
\end{align}
where $X(n)$ and $Y(n)$ are Gaussian random variables.

The aim of a channel equalization task is to construct an inverse filter, which acts on the output $y(n)$ and reproduces the original input signal as close as possible. To this end, we apply the CSVR and DRC-SVR algorithms to a set of samples of the form
\begin{align*}
\left((y(n+D), y(n+D-1), \dots, r(y+D-L+1)), s(n)\right),
\end{align*}
where $L>0$ is the filter length and $D$ the equalization time delay (in this experiment we set $L=5$ and $D=2$).

Similar to the channel identification case, in each experiment, a set of 150 pairs of samples was used to perform the training.  After training, a set of 600 pairs of samples was used to test the performance of both algorithms (i.e., to measure the mean square error between the actual input, $s(n)$, and the estimated input, $\hat s(n)$). To find the best possible values of the parameters $C$ and $t$, that minimize the mean square error for both SVR tasks, an extensive cross-validation procedure has been employed (see Tables  \ref{TAB:CSVR_equal}, \ref{TAB:DRCSVR_equal}) in a total of 100 sets of data. Figure \ref{FIG:EQUAL} shows the minimum mean square error, which has been obtained for all values of the kernel parameter $t$, versus the SVR parameter $C$, for both cases considering a circular input. Figures  \ref{FIG:EQUAL_time} and \ref{FIG:EQUAL_SV} show the computational time and the support vectors versus the MSE. The CSVR appears to achieve a slightly lower MSE for all values of the parameter $C$, at the cost of a slightly increased computational time. The results for the non-circular case are similar.

\begin{table}
\begin{center}
\begin{tabular}{|c|c|}
\hline
$C$ & t \\\hline
1 & $1/{2.5^2}$\\\hline
2 & $1/{2.5^2}$\\\hline
5 &  $1/{2.5^2}$\\\hline
10 & $1/{3^2}$\\\hline
50 & $1/{4.5^2}$\\\hline
100 & $1/{5.5^2}$\\\hline
200 & $1/{6^2}$\\\hline
500 & $1/{7^2}$\\\hline
1000 & $1/{9^2}$\\\hline
\end{tabular}
\end{center}
\caption{The values of $C$ and $t$ that minimize the mean square error of the CSVR, for the channel equalization task.}\label{TAB:CSVR_equal}
\end{table}

\begin{table}
\begin{center}
\begin{tabular}{|c|c|}
\hline
$C$ & t \\\hline
1 & $1/{1.5^2}$\\\hline
2 & $1/{1.75^2}$\\\hline
5 &  $1/{1.75^2}$\\\hline
10 & $1/{2.25^2}$\\\hline
50 & $1/{2.5^2}$\\\hline
100 & $1/{3^2}$\\\hline
200 & $1/{5^2}$\\\hline
500 & $1/{7^2}$\\\hline
1000 & $1/{7.5^2}$\\\hline
\end{tabular}
\end{center}
\caption{The values of $C$ and $t$ that minimize the mean square error of the DRC-SVR, for the channel equalization task.}\label{TAB:DRCSVR_equal}
\end{table}

\begin{figure}[t]
\begin{center}
\includegraphics[scale=0.4]{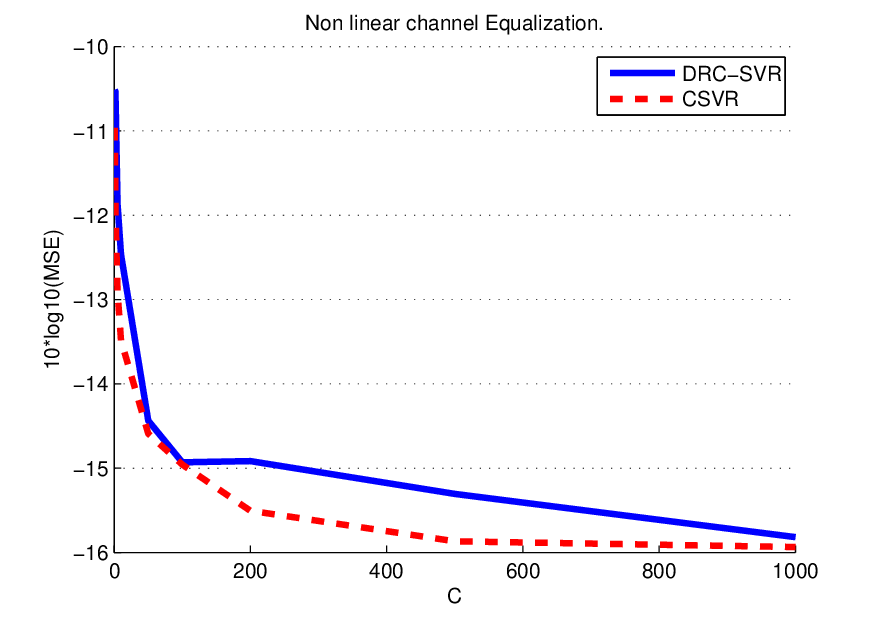}
\end{center}
\caption{MSE versus the SVR parameter $C$ for both the CSVR and the DRC-SVR rationales, for the channel equalization task.}\label{FIG:EQUAL}
\end{figure}

\begin{figure}[t]
\begin{center}
\includegraphics[scale=0.4]{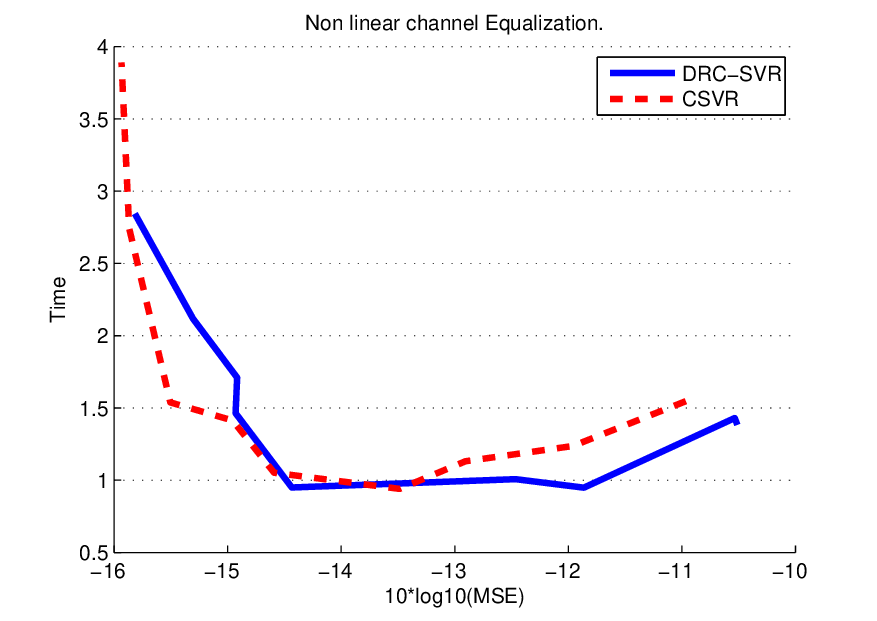}
\end{center}
\caption{Time (in seconds) versus MSE (dB) for both the CSVR and the DRC-SVR rationales, for the channel equalization task.}\label{FIG:EQUAL_time}
\end{figure}

\begin{figure}[t]
\begin{center}
\includegraphics[scale=0.4]{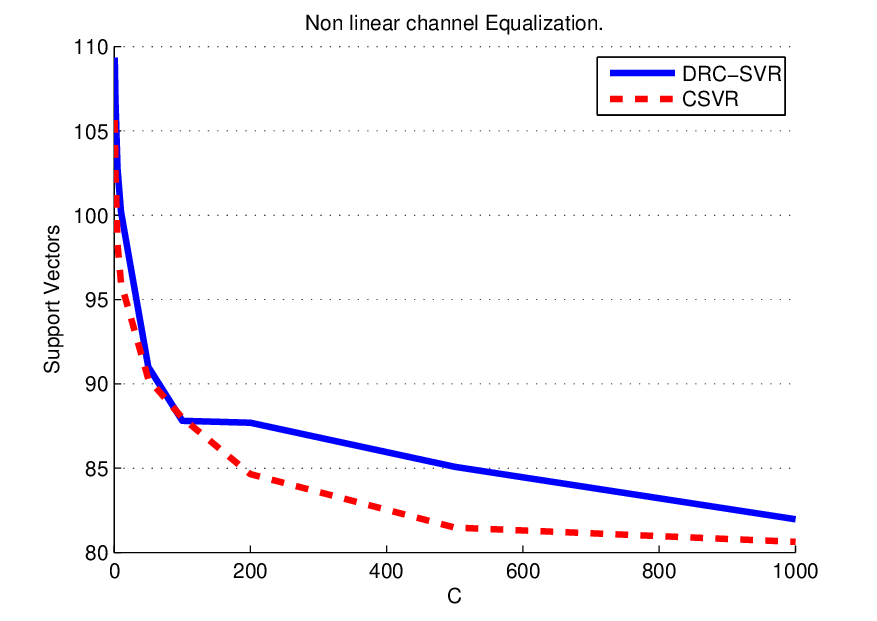}
\end{center}
\caption{The number of support Vectors versus MSE (dB) for both the CSVR and the DRC-SVR rationales, for the channel equalization task.}\label{FIG:EQUAL_SV}
\end{figure}

\subsection{SVM - Multiclass Classification}
We conclude the experimental Section with the classification case. We performed two experiments using the popular MNIST database of handwritten digits \cite{MNIST}. In both cases, the respective parameters of the SVM tasks were tuned to obtain the lowest error rate possible. The MNIST database contains 60000 handwritten digits (from 0 to 9) for training and 10000 handwritten digits for testing. Each digit is encoded as an image file with $28\times 28$ pixels. To quantify the performance of an SVM-like learning machine on the MNIST database, one typically employs a one-versus-all strategy to the training set (using the raw pixel values as input data) and then measures the success using the testing set \cite{LeCun98,DeCoste}.

\begin{table}
\begin{center}
\begin{tabular}{|c|c|c|}
\hline
 & raw data & complex Fourier coefficients\\\hline
Error rate & 3.79\% & 3.46\%\\\hline
\end{tabular}
\end{center}
\caption{The error rates of the one-versus-all classification task for the 10000 testing digits of the MNIST database.}\label{TAB:Classification1}
\end{table}

\begin{table}
\begin{center}
\begin{tabular}{|c|c|c|}
\hline
 & one versus three SVM & quaternary SVM\\\hline
Error rate & 0.721\% & 0.866\%\\\hline
\end{tabular}
\end{center}
\caption{The error rates of the one-versus-three classification task for the 0, 1, 2 and 3 testing digits of the MNIST database.}\label{TAB:Classification2}
\end{table}

In the first experiment, we compare the aforementioned standard one-versus-all scenario with a classification task that exploits complex numbers. In the complex variant, we perform a Fourier transform to each training image and keep only the 100 most significant coefficients. As these coefficients are complex numbers, we employ  a one-versus-all classification task using the binary complexified SVM rationale (see Remark \ref{REM:SVM_real_case}). In both scenarios we use the first 6000 digits of the MNIST training set (instead of the complete 60000 digits that are included in the database) to train the learning machines and test their performances using the 10000 digits of the testing set. In addition, we used the Gaussian kernel with $t=1/64$ and $t=1/140^2$, respectively. The SVM parameter $C$ has been set equal to 100 in this case as well. The error rate of the standard real-valued scenario is 3.79\%, while the error rate of the complexified (one-versus-all) SVM is 3.46\% (see Table \ref{TAB:Classification1}). In both learning tasks we used the SMO algorithm to train the SVM. The total amount of time needed to perform the training of each learning machine is almost the same for both cases (the complexified task is slightly faster).

In Section \ref{SEC:CSVM}, we discussed how the 4-classes problem is inherent to the complex SVM. Exploiting the notion of the complex pair of hyperplanes (see Figure \ref{FIG:hyperplane_C}), we have shown that the generalization of the SVM rationale to complex spaces directly assumes quaternary classification. Using this approach, the 4 classes problem can be solved using only 2 distinct SVM tasks instead of the 4 tasks needed by the one-versus-three or the one-versus-one strategies. The second experiment compares the quaternary complex SVM approach to the standard one-versus-three scenario using the first four digits (0, 1, 2 and 3) of the MNIST database. In both cases we used the first 6000 such digits of the MNIST training set to train the learning machines. We tested their performance using the respective digits that are contained in the testing set. The error rate of the one-versus-three SVM was $0.721\%$, while the error rate of the  complexified quaternary SVM was $0.866\%$ (see Table \ref{TAB:Classification2}). However, the one-versus-three SVM task required about double the time for training, compared to the complexified quaternary SVM.  This is expected, as the latter solves half as many distinct SVM tasks as the first one. In both experiments we used the Gaussian kernel with $t=1/49$ and $t=1/160^2$ respectively. The SVM parameter $C$ has been set equal to 100 in this case also.

\section{Conclusions}\label{SEC:CONCL}
A new framework for support vector regression and quaternary classification for complex data has been presented, based on the notion of the  widely-linear estimation. Both complex kernels as well as complexified real ones have been used.  We showed that this problem is equivalent to solving two separate real SVM tasks employing an induced real kernel (Figure \ref{FIG:sketch}). The induced kernel depends on the choice of the complex kernel and it is different to  the standard kernels used in the literature. Although the machinery presented here might seem similar to the dual channel approach, there are important differences. The most important one is due to the incorporation of the induced kernel $\kappa_{\C}^r$, which allows us to exploit the complex structure of the space, which is lost in the dual channel approach. As an example, we studied the complex Gaussian kernel and showed by example that the induced kernel is not the real Gaussian RBF. To the best of our knowledge, this kernel has not appeared before in the literature.  Hence, treating complex tasks directly in the complex plane, opens the way of employing novel kernels.

Furthermore, for the classification problem we have shown that the complex SVM solves directly a quaternary problem, instead of the binary problem, which is associated with the real SVM. Hence, the complex SVM not only provides the means for treating complex inputs, but also offers an alternative strategy to address multiclassification problems. In this way, such problems can be solved significantly faster (the computational time is almost the half), at the cost of increased error rate. Although, in the present work we focused on the 4 classes problem only, it is evident that the same rationale can be carried out to any multidimensional problem, where the classes must be divided into four groups each time, following a rationale similar to the one-versus-all mechanism. This will be addressed at a future time.

\bibliographystyle{IEEEtran}
%argument is your BibTeX string definitions and bibliography database(s)
\bibliography{athensBIB}

%\begin{biography}[{\includegraphics[width=1in,height=1.25in,clip,keepaspectratio]{mshell}}]{Michael Shell}

\begin{IEEEbiography}[{\includegraphics[width=1in,height=1.25in,clip,keepaspectratio]{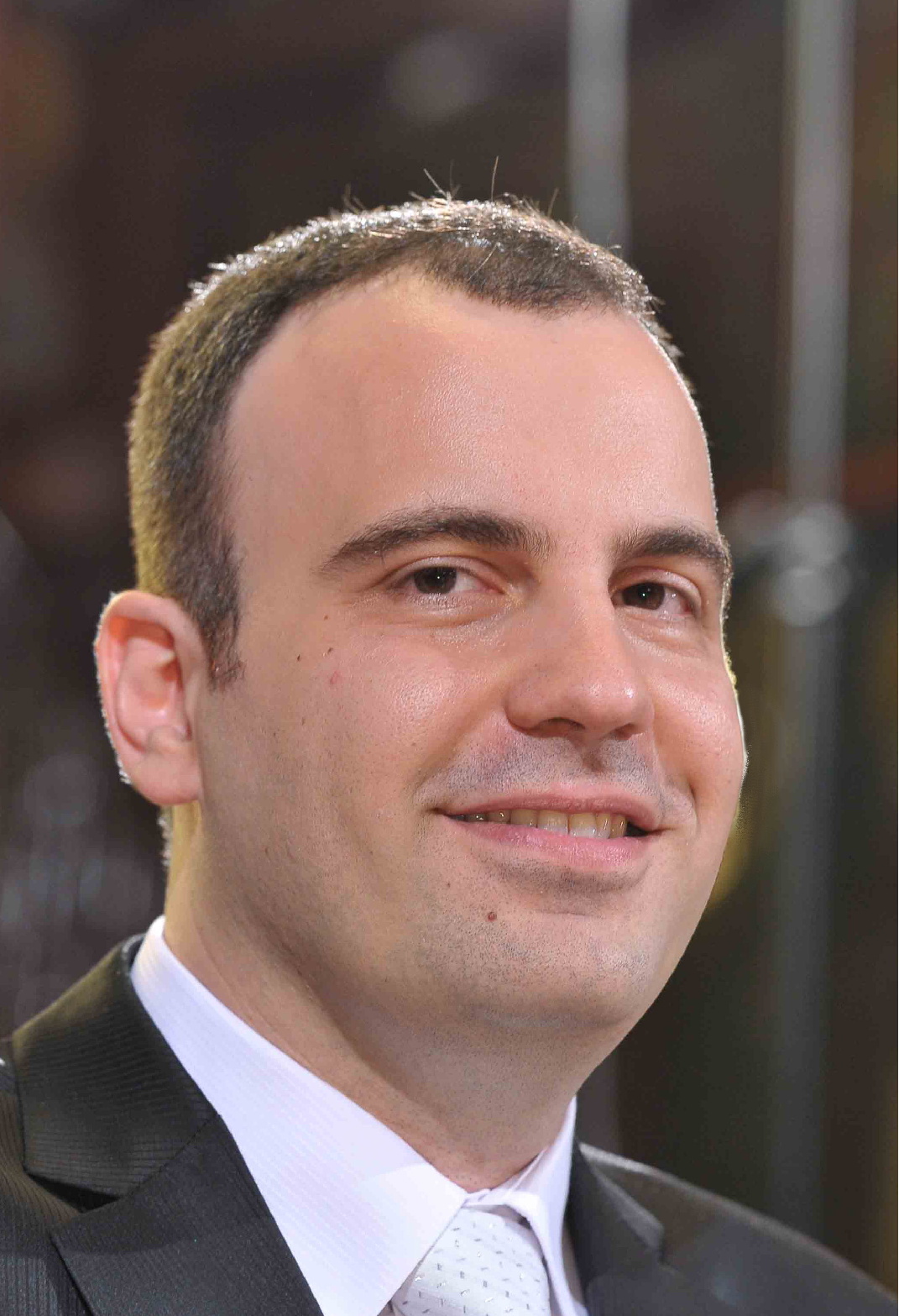}}]{Pantelis Bouboulis}(M' 10)
received the M.Sc. and Ph.D.
degrees in Informatics and Telecommunications from 	
the National and Kapodistrian University of Athens,
Greece, in 2002 and 2006, respectively. 	
From 2007 till 2008, he served as an Assistant Professor in the Department of Informatics and Telecommunications,
University of Athens. Currently, he teaches mathematics at the Zanneio Model Experimental Lyceum of Pireas and works
as a research assistant at the Signal and Image processing laboratory of the department of Informatics
and Telecommunications  of the University of Athens. From 2012, he serves as an Associate Editor of the
IEEE Transactions of Neural Networks and Learning Systems.
His current research interests lie in the areas of machine learning, 	
fractals, wavelets and image processing.
\end{IEEEbiography}

\begin{IEEEbiography}[{\includegraphics[width=1in,height=1.25in,clip,keepaspectratio]{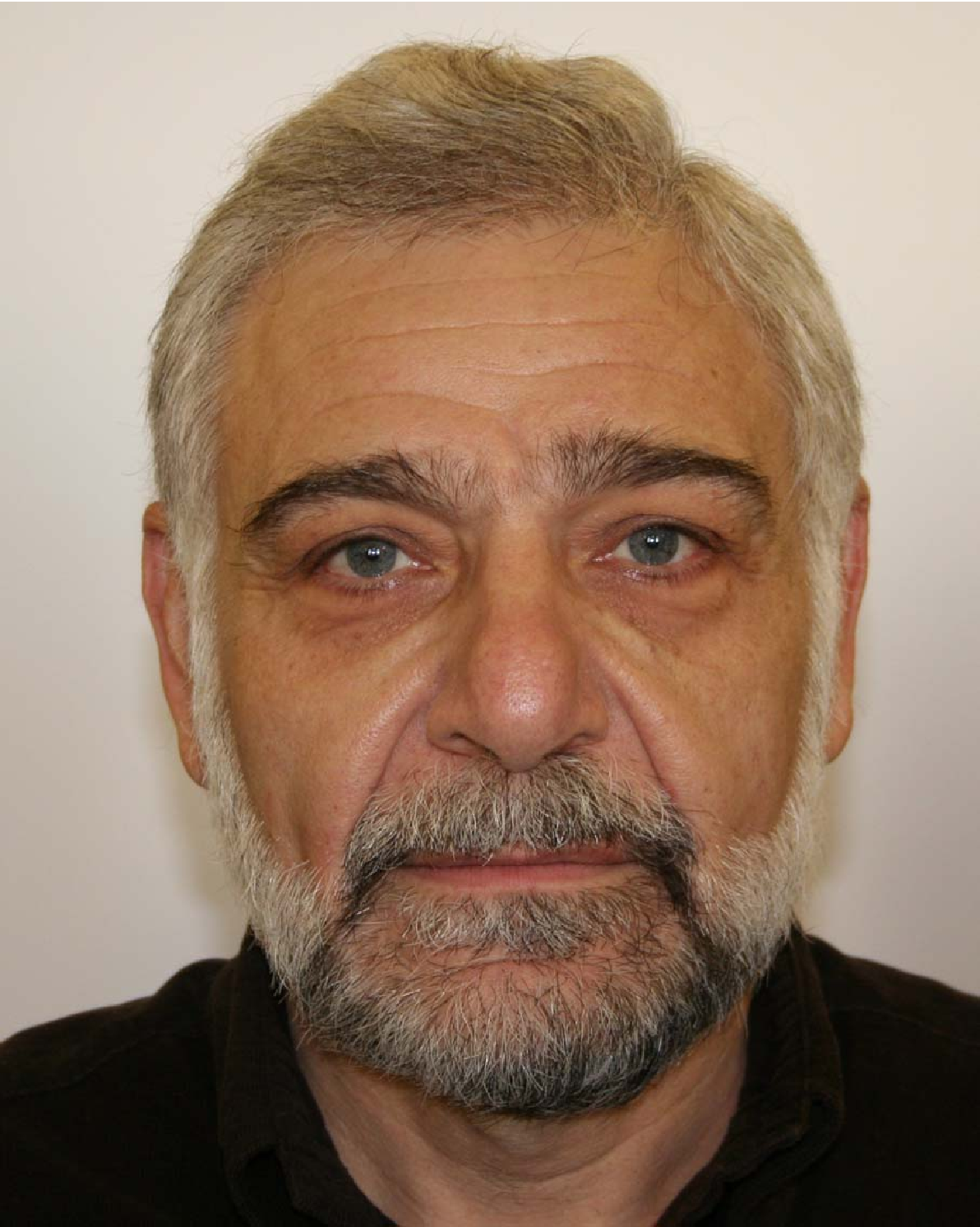}}]{Sergios Theodoridis}(Fellow)
is currently Professor of Signal Processing and Machine Learning in the
Department of Informatics and Telecommunications of the University of Athens. His research
interests lie in the areas of Adaptive Algorithms, Distributed and Sparsity-Aware Learning,
Machine Learning and Pattern Recognition, Signal Processing for Audio Processing and
Retrieval. He is the co-editor of the book "Efficient Algorithms for Signal Processing and System
Identification", Prentice Hall 1993, the co-author of the best selling book "Pattern Recognition",
Academic Press, 4th ed. 2008, the co-author of the book "Introduction to Pattern Recognition: A
MATLAB Approach", Academic Press, 2009, and the co-author of three books in Greek, two of
them for the Greek Open University.
He serves as Editor-in-Chief for the IEEE Transactions on Signal Processing (2015-2018). He is
Editor-in-Chief for the Signal Processing Book Series, Academic Press and co-Editor in Chief
(with Rama Chellapa) for the E-Reference Signal Processing, Elsevier.
He is the co-author of six papers that have received best paper awards including the 2009 IEEE
Computational Intelligence Society Transactions on Neural Networks Outstanding paper Award.
He is the recipient of the EURASIP 2014 Meritorious Service Award. He has served as an IEEE
Signal Processing Society Distinguished Lecturer. He was Otto Monstead Guest Professor,
Technical University of Denmark, 2012, and holder of the Excellence Chair, Dept. of Signal
Processing and Communications, University Carlos III, Madrid, Spain, 2011. He currently serves
as Distinguished Lecturer for the IEEE Circuits and Systems Society.
He was the general chairman of EUSIPCO-98, the Technical Program co-chair for ISCAS-2006
and ISCAS-2013, co-chairman and co-founder of CIP-2008, co-chairman of CIP-2010 and
Technical Program co-chair of ISCCSP-2014. He has served as President of the European
Association for Signal Processing (EURASIP), as a member of the Board of Governors for the
IEEE CAS Society and as a member of the Board of Governors (Member-at-Large) of the IEEE
SP Society.
He has served as a member of the Greek National Council for Research and Technology and he
was Chairman of the SP advisory committee for the Edinburgh Research Partnership (ERP). He
has served as vice chairman of the Greek Pedagogical Institute and he was for four years member
of the Board of Directors of COSMOTE (the Greek mobile phone operating company). He is
Fellow of IET, a Corresponding Fellow of the Royal Society of Edinburgh (RSE), a Fellow of
EURASIP and a Fellow of IEEE.
\end{IEEEbiography}

\begin{IEEEbiography}[{\includegraphics[width=1in,height=1.25in,clip,keepaspectratio]{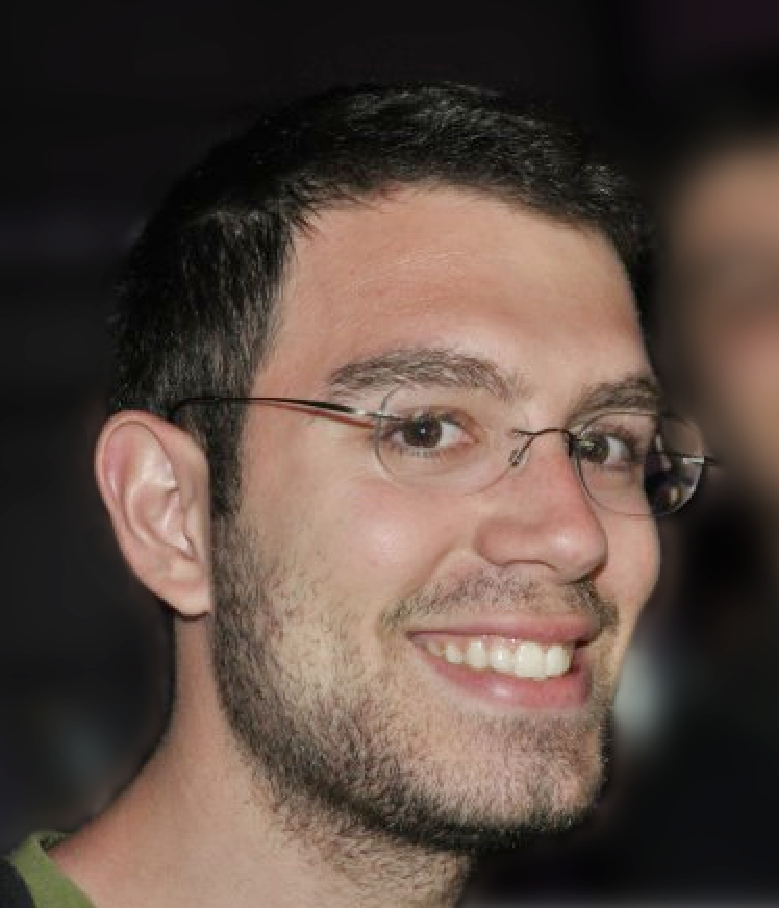}}]{Charalampos Mavroforakis}
received his B.Sc. degree in Computer Science from Athens University
of Economics and Business, Greece, in 2010. Currently, he is a PhD student in the Computer
Science department at Boston University. His research interests include data mining,
(social) network analysis and recommendation systems.
\end{IEEEbiography}

\begin{IEEEbiography}[{\includegraphics[width=1in,height=1.25in,clip,keepaspectratio]{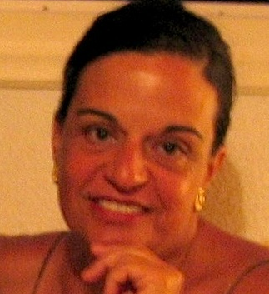}}]{Leonni Dalla}
received the M.Sc. degree from Birkbeck College at London University and the Ph.D degree
from University College London (UCL). Currently, she serves as an Associate Professor  at the department of Mathematics
of the National and Kapodistrian University of Athens, where she has taught several undergraduate and post-graduate courses.
From 2001 to 2002 and from 2003 to 2004 she served as a visiting professor at the University of Cyprus. Her current
research interests lie in the areas of Convex Geometry, fractals and Mathematical Analysis.
\end{IEEEbiography}

\end{document}